\documentclass[onecolumn,11pt]{article}
\usepackage{amssymb}
\usepackage{amsmath}
\usepackage{amsfonts}
\usepackage{graphics}
\usepackage{epsf}
\usepackage{float}
\usepackage[pdftex]{graphicx}
\usepackage{amsthm}
\usepackage{algorithm}
\usepackage{multirow}
\usepackage{authblk}
\usepackage{url}
\usepackage{psfrag}

\usepackage[lofdepth,lotdepth]{subfig}

\usepackage{xcolor}
\usepackage{booktabs}

\newtheorem{theorem}{Theorem}

\newtheorem{lemma}[theorem]{Lemma}

\newtheorem{remark}{Remark}
\DeclareMathOperator*{\argmin}{arg\,min}
\DeclareMathOperator*{\argmax}{arg\,max}

\begin{document}

\title{\bf A Representation Theory Perspective on Simultaneous Alignment and Classification}
\author{Roy R. Lederman\thanks{roy@math.princeton.edu, The Program in Applied and Computational Mathematics, Princeton University, Princeton, NJ, USA} and Amit Singer\thanks{amits@math.princeton.edu, The Department of Mathematics and The Program in Applied and Computational Mathematics, Princeton University, Princeton, NJ, USA} }

\maketitle

\begin{abstract}

One of the difficulties in 3D reconstruction of molecules from images in single particle Cryo-Electron Microscopy (Cryo-EM), in addition to high levels of noise and unknown image orientations, is heterogeneity in samples: in many cases, the samples contain a mixture of molecules, or multiple conformations of one molecule.
Many algorithms for the reconstruction of molecules from images in heterogeneous Cryo-EM experiments are based on iterative approximations of the molecules in a non-convex optimization that is prone to reaching suboptimal local minima. Other algorithms require an alignment in order to perform classification, or vice versa. 
The recently introduced Non-Unique Games framework provides a representation theoretic approach to studying problems of alignment over compact groups, and offers convex relaxations for alignment problems which are formulated as semidefinite programs (SDPs) with certificates of global optimality under certain circumstances.
In this manuscript, we propose to extend Non-Unique Games to the problem of simultaneous alignment and classification with the goal of simultaneously classifying Cryo-EM images and aligning them within their respective classes. 
Our proposed approach can also be extended to the case of continuous heterogeneity.

\end{abstract}

%
%
%
\section{Introduction}\label{sec:intro}

A Non-Unique Game (NUG) is an optimization problem or a statistical estimation problem, 
of inferring $n$ elements of a group $g_1,\ldots ,g_n \in \mathcal{G}$ by minimizing an expression
of the form 
\begin{equation}\label{eq:NUGdef}
  \underset{g_1,\ldots ,g_n \in \mathcal{G}}{\argmin} \sum_{i,j=1}^n f_{ij}(g_i g_j^{-1}) .
\end{equation}
where $f_{ij}: \mathcal{G} \rightarrow \mathbb{R}$ are penalty functions for particular pairwise relations $g_i g_j^{-1}$ between elements. 
This problem arises in Multireference Alignment discussed in \cite{bandeira2014multireference},
and in more general settings discussed in \cite{bandeira2015non}; 
A convex relaxation of the problem, proposed in  \cite{bandeira2015non}, can be solved using Semidefinite Programming (SDP).
One of the applications where NUGs and associated algorithms have been of particular interest
is Cryo-electron microscopy (Cryo-EM) \cite{frank2006three,van2000single}, where multiple noisy 2D projections (images) from unknown directions of an unknown 3D molecule must be aligned over SO(3), as a step in reconstructing the molecule. 
 
Cryo-EM has been named Method of the Year 2015 by the journal Nature Methods due to the breakthroughs that the method facilitated in mapping the structure of molecules that are difficult to crystallize. 
One of the difficulties in Cryo-EM, which has been noted, for example, in the surveys accompanying the Nature Methods announcement \cite{doerr2016single,nogales2016development,glaeser2016good}, is heterogeneity in the sample:
in practice many samples contain two (or more) distinct types of molecules (or different conformations of the same molecule). 
Algorithms for Cryo-EM processing in the presence of heterogeneity (for example, \cite{herman2008classification,shatsky2010automated,scheres2010chapter,penczek2011identifying,katsevich2015covariance,anden2015covariance}) must therefore determine both the class of each image, and its alignment with respect to other images of the same class;
this often requires some initial educated guess of the structure of the molecules in the sample, iterative estimations of the structure, alignment and classification, 
or some method of performing one of the two tasks of alignment and classification before the other task.

In this manuscript we propose to solve the classification and alignment problems simultaneously.
This approach is based on the observation that both alignment and classification are problems over compact groups, and that the direct product of these groups is also a compact group. 

We reformulate the problem as an optimization problem over the direct product of the groups,
and reduce it to a NUG. 
In addition, we discuss some of the symmetries in the problem, which are exploited to reduce the size of the optimization problem. 
Furthermore, we propose an approach for controlling the size of the classes.

The approach can be generalized to simultaneous alignment and parametrization, in the case of continuous heterogeneity (which will be discussed in a future paper).

This manuscript is organized as follows.
Section \ref{sec:pre} summarizes some standard results used in this manuscript, as well as some previous work on NUGs. 
Section \ref{sec:analysis} contains a more detailed description of the problem and applications, 
and the derivation of the main arguments in this manuscript.
In Section \ref{sec:alg} we propose algorithms for simultaneous alignment and classification.
Section \ref{sec:res} contains experimental results for the case of simultaneous alignment and classification on SO(2). 
In Section \ref{sec:conclusion} we summarize our conclusions and briefly discuss generalizations and future work.

\section{Preliminaries}\label{sec:pre}

\subsection{The Cryo-EM Problem}

Electron Microscopy is an important tool for recovering the 3D structure of molecules.
Of particular interest in the context of this manuscript is Single Particle Reconstruction (SPR), 
and more specifically, Cryo-EM,
where multiple noisy 2D projections, ideally of identical particles in different orientations,
are used in order to recover the 3D structure.
The following formula is a simplified imaging model of SPR
\begin{equation}\label{eq:cryoprojection:model}
  \left( \mathcal{P}_R\mathcal{X} \right)(x,y) = \int_{z} \mathcal{X}(Rr) dz
\end{equation}
where $r=(x,y,z)$, $R$ is some random rotation matrix in SO(3), $\mathcal{X}$ is the scattering density of the molecule, and $\mathcal{P}$ is the projection operator.
In other words, the model is that the molecule is rotated in a random direction, and the image obtained is the top-view projection of the rotated molecule, integrating out the $z$ axis.
Indeed, one of the defining properties of SPR and Cryo-EM is that the orientation $R$ of the molecule in each image is unknown, unlike other tomography techniques, where the rotation angles are recorded with the measurements. 
The analysis of Cryo-EM images is further complicated by extremely high levels of noise, far exceeding the signal in magnitude, which makes it difficult not only to analyze the particles in the images, but also to locate the particles in the micrographs produced.
Sample images are presented in Figure \ref{fig:cryosample}.
More detailed discussions of these challenges, and various other challenges, such as the contrast transfer functions (CTF) applied to the images in the imaging process, can be found, for example, in \cite{frank2006three}.

The reconstruction of the molecule (or, more precisely, the density  $\mathcal{X}$) from the images obtained in Cryo-EM requires an estimate of the rotation angles of the images. 
The Fourier Slice Theorem (see, for example, \cite{natterer1986mathematics}) provides a way to estimate these rotations from the common lines between the images (see, for example \cite{van1987angular,singer2010detecting,singer2011three,shkolnisky2012viewing}, and Figure \ref{fig:common}).
In the context of this manuscript, we assume that for every pair of images $i$ and $j$, 
we have some function $f_{ij}(g)$ which corresponds to the ``incompatibility'' between the images $i$ and $j$ for every relative orientation $g \in SO(3)$; this function is a measure of the discrepancy  between the radial line in the Fourier transform of image $i$ and the radial line in the Fourier transform of image $j$ which would have corresponded to the common line between the plane of $i$ and the plane of $j$, if the relative orientation of the two images had been $g$.
Had there not been noise, we would have expected that $f_{ij}(g_{ij})=0$ for the true relative orientation $g_{ij}$ between image $i$ and image $j$,  and $f_{ij}(g) > 0$ for every other $g$ (in fact, $f_{ij}(\tilde{g}_{ij})=0$ for every $\tilde{g}_{ij}$ that yields the same common lines for the pair of images as $g_{ij}$ since various rotations can yield the same common line. The ambiguity is resolved, up to reflections, only by adding a third image). 
In practice, due to the high levels of noise, $f_{ij}$ need not be $0$ at $g_{ij}$, and in fact, the value of $f_{ij}$ may even not be minimized at $g_{ij}$. 
However, the expected value of $f_{ij}$ is  lower for the true $g_{ij}$ than it is for other relative rotations. 
For more details about this ``penalty'' function in the context of this manuscript, see \cite{bandeira2015non}.

\begin{figure}
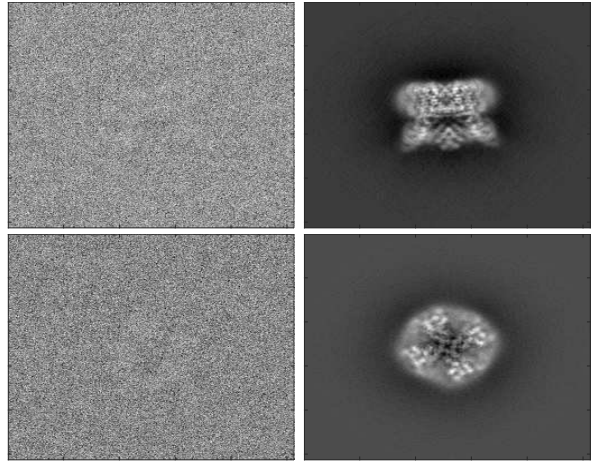

\begin{center}
\includegraphics[width=1.5in,trim={0.3in 0.15in 0 0},clip]{figures/raw1_tv}%
~\includegraphics[width=1.5in,trim={0.3in 0.15in 0 0},clip]{figures/closest1_tv}%

\includegraphics[width=1.5in,trim={0.3in 0.15in 0 0},clip]{figures/raw2_tv}%
~\includegraphics[width=1.5in,trim={0.3in 0.15in 0 0},clip]{figures/closest2_tv}%
\end{center}
\caption{Left: two raw experimental images of  TRPV1, available via EMDB 5778 \cite{liao2013structure}. Right: computed projections of TRPV1 which are the closest to the images on their left. }\label{fig:cryosample}
\end{figure}

\begin{figure}
\begin{center}
\includegraphics[trim=1in 0 0 0]{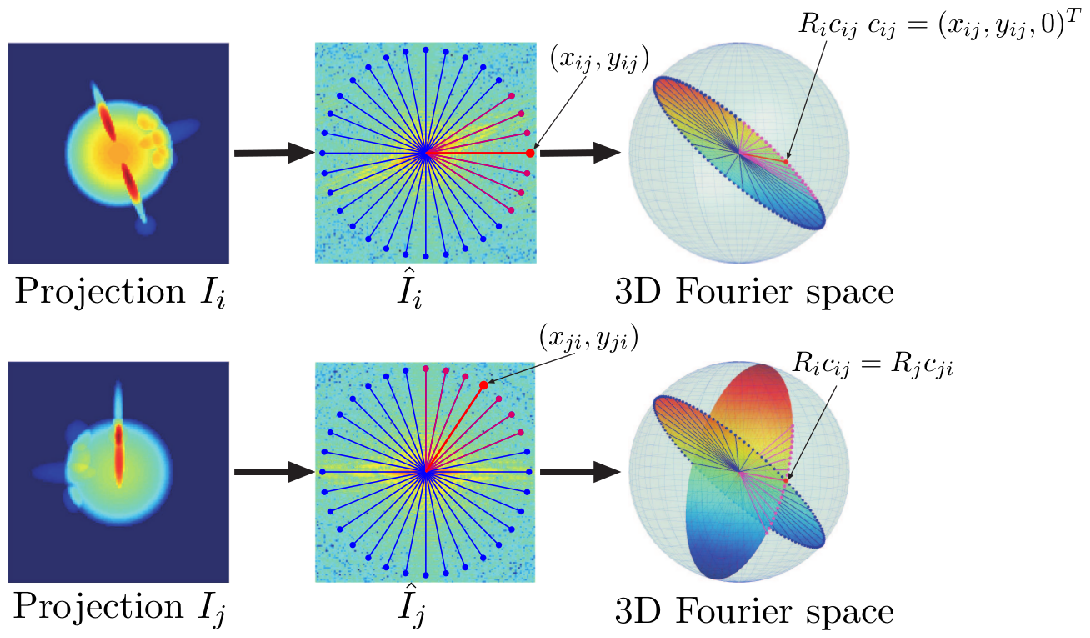}%
\end{center}
\caption{Common Lines in Cryo-EM. The left most images $I_i$ and $I_j$ are examples of projections of a molecule density $\mathcal{X}$; each projection is obtained from a different direction. At the center, are the Fourier transforms $\hat{I}_i$ and $\hat{I}_j$ of those images, overlaid with radial lines. 
The lower left sub-figure is a visualization of the two slices of the 3D Fourier transform of the 3D density $\mathcal{X}$, corresponding to $\hat{I}_i$ and $\hat{I}_j$; the two slices intersect each other, so that there is a line in $\hat{I}_i$ that is identical to a line in $\hat{I}_j$ (assuming no noise).
Indeed, the point $(x_{ij},y_{ij})$ which lies along this common line in $\hat{I}_i$ is identical to the point $(x_{ji},y_{ji})$ which lies along this common line in $\hat{I}_j$.
A more detailed discussion of common lines is available, for example, in \cite{van1987angular,singer2010detecting,singer2011three,shkolnisky2012viewing} }\label{fig:common}.
\end{figure}

\subsection{The Heterogeneity Problem in Cryo-EM}

So far, we have assumed that all the molecules being imaged in an experiment are identical copies of each other, so that all the images are projections of identical copies, from different directions. 
However, in practice, the molecules in a given sample may differ from one another for various reasons. 
For example, the sample may contain several types of different molecules due to some contamination or feature of the experiment. Alternatively, the molecules which are studied may have several different conformations or states, or some local variability (see example in Figure \ref{fig:cryosample:heterogeneity}). 
The heterogeneity may be discrete (e.g. for distinct different molecules) or continuous (for molecules with continuous variability). 

When there is heterogeneity in the samples, high resolution reconstruction of the molecules requires not only an estimate  of the  rotation of each image, but also classification of the images into clusters, each corresponding to a different molecule which is to be reconstructed separately. 
Some of the existing SPR analysis methods rely on some prior knowledge of the underlying molecules and on iterative processes of estimating the structure of the molecules and matching images to those estimates (e.g. \cite{sigworth2010chapter,scheres2010chapter}), and others require some method of recovering the rotation of the images although the images reflect mixtures of projections of different molecules (e.g. \cite{katsevich2015covariance,anden2015covariance}).

\begin{figure}
\begin{center}
\includegraphics[width=4in]{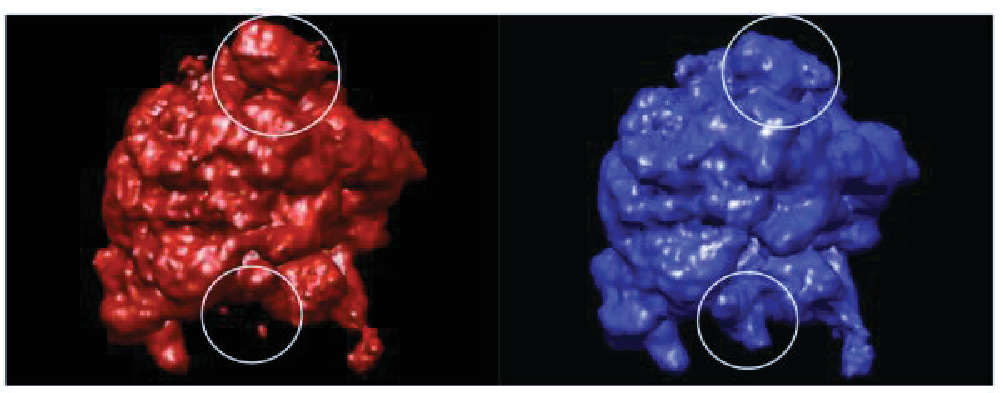}%
\end{center}
\caption{Classical (left) and hybrid (right) states of 70S E. Coli ribosome (image source: \cite{liao2010classification}). }\label{fig:cryosample:heterogeneity}
\end{figure}

\clearpage

%
%
%

\begin{table}[h!]\caption{Table of Notation}
\begin{center}
\begin{tabular}{r c p{10cm} }
\toprule
$ A^*$ & ~ &  the complex conjugate transpose of the matrix $A$\\
$ \mathbb{Z}_M$ & ~ &  the cyclic group of order $M$\\
$ \mathcal{G} \times \mathcal{A} $ & ~ &  the direct product between group $\mathcal{G}$ and group $\mathcal{A}$ \\
$ g \circ f$ & ~ &  the action of $g\in \mathcal{G}$ on a function $f \in L^2(\mathcal{Y})$: $(g \circ f)(x) = f(g^{-1}x)$ \\ 
$tr(A)$ & ~ & the trace of the matrix $A$  \\
$A \otimes B$ & ~ & the Kronecker (tensor) product of the matrix $A$ and the matrix $B$ \\
\bottomrule
\end{tabular}
\end{center}
\label{tb:notation}
\end{table}

\subsection{Irreducible Representations of Groups}\label{sec:pre:irred}

The purpose of sections \ref{sec:pre:irred}, \ref{sec:pre:zm} and \ref{sec:pre:prod}  is to briefly review some standard results in group theory and harmonic analysis; more detailed discussions of these facts can be found, for example, in \cite{coifman1968representations,s1995group,dym1985fourier}.

Suppose that $\mathcal{G}$ is a compact group and $f\in L^2(\mathcal{G})$, then
by the Peter-Weyl Theorem \cite{peter1927vollstandigkeit}, the generalized Fourier expansion of $f$ is
\begin{equation}\label{eq:Peter-Weyl}
  f(g) = \sum_k d_k tr \left( \hat{f}^{(k)} \rho_{k}(g) \right),
\end{equation}
where the matrices $\rho_{k}(g)$ are the irreducible representations of $\mathcal{G}$,
$d_k$ is the dimensionality of the $k$th representation, 
and the matrices $\hat{f}^{(k)}$ are the Fourier coefficients of $f$, defined by the formula
\begin{equation}\label{eq:Peter-Weyl:coeff}
  \hat{f}^{(k)} = \int_{\mathcal{G}} f(g) \rho_{k}^*(g) dg ,
\end{equation}
with $dg$ the Haar measure on $\mathcal{G}$ normalized so that
\begin{equation}\label{eq:haar:norm}
  \int_{\mathcal{G}} dg =1 .
\end{equation}

\begin{remark}\label{rm:abelian:dim}
For abelian groups, such as SO(2) (shifts on a circle), $d_k=1$ for all $k$. 
However, in SO(3), which is of particular interest in the Cryo-EM application, $d_k = 2k+1$ with $k=0,1,2,\ldots$.
\end{remark}

The integration of any irreducible representation with respect to the Haar measure yields the zero matrix, 
except for the case of the trivial constant irreducible representation $\rho_{0}$:
\begin{equation}\label{eq:haar:avg}
  \int_\mathcal{G} \rho_{k}(g) dg = 0 ~~~~~ \forall k \ne 0 .
\end{equation}

The following are well known properties of irreducible and unitary representations of compact groups:
\begin{equation}\label{eq:irred:prod}
  \rho_{k}(g_1 g_2) = \rho_{k}(g_1 ) \rho_{k}(g_2) ,
\end{equation}

\begin{equation}\label{eq:irred:inv}
  \rho_{k}(g^{-1}) = \rho_{k}^*(g) .
\end{equation}

\subsection{Special Cases: SO(2) and $\mathbb{Z}_M$}\label{sec:pre:zm}

In the special case where $\mathcal{G} = \mathbb{Z}_M$ (discrete cyclic group of $M$ elements),
there is a finite set of $M$ irreducible representations, 
and all the irreducible representations are of dimensionality one (scalar rather than a matrix).
The irreducible representations $\{\eta_m\}_{m=0}^{M-1}$ of $\mathbb{Z}_M$ are
\begin{equation}
\eta_m(a) = e^{\mathrm{i} 2 \pi a m/M} ~~,~~ a=0,1,\ldots,m-1.
\end{equation}
The Fourier coefficients of a function over $\mathbb{Z}_M$ are simply the discrete Fourier transform (DFT) of the function (with the appropriate normalization (\ref{eq:haar:norm})).

In the special case where $\mathcal{G} = SO(2)$,
there is an infinite set of irreducible representations, 
and all the irreducible representations are of dimensionality one.
The irreducible representations $\{\eta_k\}_{k=-\infty}^{\infty}$ of SO(2) are
\begin{equation}\label{eq:irred:so2}
\eta_k(a) = e^{\mathrm{i} a k} ~~,~~ a \in [0, 2\pi) .
\end{equation}

\begin{remark}
For the sake of brevity, and with a small abuse of notation, we will use elements of the groups 
$\mathbb{Z}_M$ and SO(2) and integers and angles interchangeably.
For example, in (\ref{eq:irred:so2}), the variable ``$a$'' can denote an element of SO(2) or an angle.
Therefore,  $a_1 a_2^{-1}$ would mean the same as $a_1 - a_2$, with the former in group notation and the latter in angle notation; $a = \mathrm{e}$ (where $\mathrm{e}$ is the identity element) in group notation means the same as $a=0$ in angle notations.
The appropriate interpretation, group element or integers and angles, is obvious from the context or does not matter. 
\end{remark}

\subsection{Direct Products of Groups}\label{sec:pre:prod}

The direct product $\mathcal{G} \times \mathcal{A}$ of two compact groups $\mathcal{G}$ and $\mathcal{A}$ is also a compact group, which  has the elements 
$\{(g,a): g\in\mathcal{G}, a\in \mathcal{A}\}$. 
In this manuscript, we are particularly interested in the case $\mathcal{A} = \mathbb{Z}_M$.

The product of two elements of  $\mathcal{G} \times \mathcal{A}$ is defined 
in terms of elements in $\mathcal{G}$ and $\mathcal{A}$ by the following formula
\begin{equation}\label{eq:groupprod:act}
(g_i,a_i)  (g_j,a_j) = \left( g_i  g_j , a_i  a_j \right).
\end{equation}
It follows that 
\begin{equation}\label{eq:groupprod:inv}
(g_i,a_i)  (g_j,a_j)^{-1} = \left( g_i  g_j^{-1} , a_i  a_j^{-1} \right) .
\end{equation}

If $\eta_m(a)$ is an irreducible representation of $\mathcal{A}$ and $\rho_k(g)$ is an irreducible representation of $\mathcal{G}$,
then $\psi_{k,m}( (g,a) )$, defined by the formula
\begin{equation}
\psi_{k,m}( (g,a) ) = \rho_m(g) \otimes \eta_m(a),
\end{equation} 
is an irreducible representation of $\mathcal{G} \times \mathcal{A}$.
The irreducible representations $\psi_{k,m}( (g,a) )$ of $\mathcal{G} \times \mathbb{Z}_M$ are summarized in Table \ref{tb:prodrep:general};
in Table \ref{tb:prodrep:trivial} we substitute  $\eta_0(a)=1$ and $\rho_0(g)=1$  for the trivial irreducible representations of $\mathcal{A}$ and $\mathcal{G}$ respectively.
By Remark \ref{rm:abelian:dim}, the irreducible representations of abelian groups, like the irreducible representations $\eta_m$ of $\mathbb{Z}_M$, are one dimensional, so in this special case, the tensor product $\otimes$ can be replaced with the trivial 
product between the scalar valued function $\eta_m(a)$ and the (possibly) matrix valued function $\rho_k(g)$, as summarized in Table \ref{tb:directprod:aligncluster:ireps}.
  
  \begin{table}[h]
    \centering
  \begin{tabular}{ l l || c | c | c }
    $\psi_{k,m}\left( (g,a) \right)$       &  & $m$=0     & $m$=1     & $\cdots$          \\			
              &   & $\eta_0(a)$      & $\eta_1(a)$      & $\cdots$          \\
    \hline \hline 
  $k=0$ & $\rho_0(g)$ & $\rho_0(g)\otimes \eta_0(a)$ & $ \rho_0(g) \otimes \eta_1(a)$  & $\cdots$          \\
    $k=1$ & $\rho_1(g)$ & $\rho_1(g) \otimes \eta_0(a)$ & $ \rho_1(g) \otimes \eta_1(a)$  & $\cdots$      \\
    $k=2$ & $\rho_2(g)$ & $\rho_2(g) \otimes \eta_0(a)$ & $ \rho_2(g) \otimes \eta_1(a)$  & $\cdots$       \\
    $k=3$ & $\rho_3(g)$ & $\rho_3(g) \otimes \eta_0(a)$ & $ \rho_3(g) \otimes \eta_1(a)$  & $\cdots$   \\
    \vdots & \vdots & \vdots & $\ddots$  \\
  \end{tabular}
  \caption{Irreducible representations of $\mathcal{G} \times \mathcal{A}$ }\label{tb:prodrep:general}
  \end{table}

  \begin{table}[h]
    \centering
  \begin{tabular}{ l || c | c | c }			
    $\psi_{k,m}\left( (g,a) \right)$             & $\eta_0(a)=1$      & $\eta_1(a)$      & $\cdots$          \\
    \hline \hline 
 $\rho_0(g)=1$ &  1 &   $\eta_1(a)$  & $\cdots$          \\
    $\rho_1(g)$ & $\rho_1(g)$ & $ \rho_1(g) \otimes \eta_1(a)$  & $\cdots$      \\
    $\rho_2(g)$ & $\rho_2(g)$ & $ \rho_2(g) \otimes \eta_1(a)$  & $\cdots$       \\
     $\rho_3(g)$ & $\rho_3(g)$ & $ \rho_3(g) \otimes \eta_1(a)$  & $\cdots$   \\
    \vdots & \vdots & \vdots & $\ddots$  \\
  \end{tabular}
  \caption{Product irreducible representations, after substituting the trivial irreducible representations}\label{tb:prodrep:trivial}\label{tb:prodrep:scalar}
  \end{table}

  \begin{table}[h]
    \centering
  \begin{tabular}{ l || c | c | c | c}			
    $\psi_{k,m}\left( (g,a) \right)$             & $\eta_0(a)=1$      & $\eta_1(a)$      & $\cdots$    & $\eta_{M-1}(a)$      \\
    \hline \hline 
 $\rho_0(g)=1$ &  1 &   $\eta_1(a)$  & $\cdots$  & $\eta_{M-1}(a)$        \\
    $\rho_1(g)$ & $\rho_1(g)$ & $ \rho_1(g)  \eta_1(a)$  & $\cdots$ & $ \rho_1(g)  \eta_{M-1}(a)$    \\
    $\rho_2(g)$ & $\rho_2(g)$ & $ \rho_2(g)  \eta_1(a)$  & $\cdots$ & $ \rho_2(g)  \eta_{M-1}(a)$      \\
     $\rho_3(g)$ & $\rho_3(g)$ & $ \rho_3(g) \eta_1(a)$  & $\cdots$ & $ \rho_3(g)  \eta_{M-1}(a)$  \\
    \vdots & \vdots & \vdots & $\ddots$  & \vdots \\
  \end{tabular}
  \caption{Product irreducible representations in the special case of $\mathcal{G} \times \mathbb{Z}_M$, after plugging in the trivial irreducible representations}\label{tb:directprod:aligncluster:ireps}
  \end{table}

\clearpage
\subsection{Non-Unique Games (NUG)}

Let $\mathcal{G}$ be a compact group, and for every $1 \le i,j \le n$ let $f_{ij} \in L^2(\mathcal{G})$;  Non-Unique Games (NUG) are problems of the form (\ref{eq:NUGdef}).

\begin{remark}\label{rem:unique}
The solutions to Non-Unique Games are not unique, in the sense that if 
$g_1,\ldots ,g_n$ is a solution, then, $g_1g,\ldots ,g_ng$ is also a solution 
for any $g\in \mathcal{G}$, because $f_{ij} \left( g_ig (g_jg)^{-1}\right) =  f_{ij} \left( g_i g_j^{-1}\right)$.
The solution is therefore unique at most up to a global group element; the relative pairwise ratios  $g_{i}g_{j}^{-1}$ may be  unique.
\end{remark}

\subsubsection{Fourier Expansion of a NUG, and a Matrix Form}

Using the Fourier expansion (see (\ref{eq:Peter-Weyl})) of $f_{ij}$,
\begin{equation}
  f_{ij}(g_i g_j^{-1}) = \sum_{k=0}^\infty d_k tr \left( \hat{f}_{ij} \rho_k(g_i g_j^{-1}) \right) ,
\end{equation}
we rephrase (\ref{eq:NUGdef}) in the {\bf Fourier expansion form}: 
  \begin{equation}\label{eq:NUGdef:Fourier}
  \underset{g_1,\ldots ,g_n \in \mathcal{G}}{\argmin } \sum_{i,j=1}^n \sum_{k=0}^\infty d_k tr \left( \hat{f}_{ij}^{(k)} \rho_k(g_i g_j^{-1}) \right) .
  \end{equation}
For example, in the case of $\mathbb{Z}_M$, the Fourier coefficients of $f_{ij}$ are given by its DFT, and the NUG becomes
  \begin{equation}
  \underset{a_1,\ldots ,a_n \in \mathbb{Z}}{\argmin} \sum_{i,j=1}^n \sum_{m=0}^{M-1}  \hat{f}^{(k)}_{ij} e^{\mathrm{i} 2\pi m (a_i - a_j)/M}  .
  \end{equation}
Plugging (\ref{eq:irred:prod}) into (\ref{eq:NUGdef:Fourier}) yields
  \begin{equation}
  \underset{g_1,\ldots,g_n \in \mathcal{G}}{\argmin} \sum_{i,j=1}^n \sum_{k=0}^\infty d_k tr \left( \hat{f}^{(k)}_{ij} \rho_k (g_i) \rho_k^{*}(g_j) \right) .
  \end{equation}
The same expression can be rewritten in a {\bf block matrix form}:
  \begin{equation}\label{eq:NUG:Fourier:Block}
  \underset{g_1,\ldots,g_n \in \mathcal{G}}{\argmin} \sum_{k=0}^\infty tr \left( \hat{F}^{(k)} X^{(k)} \right) ,
  \end{equation}
where,  
  \begin{equation}\label{eq:NUG:Fourier:Block:structure1}
    \begin{array}{cc}
    X^{(k)} =  \begin{bmatrix} \rho_k (g_1) \\ \vdots \\ \rho_k (g_n) \end{bmatrix} \begin{bmatrix} \rho_k (g_1) \\ \vdots \\ \rho_k (g_n) \end{bmatrix}^* , 
    &
    \hat{F}^{(k)}= d_k \begin{bmatrix}  \hat{f}^{(k)}_{11} & \cdots & \hat{f}^{(k)}_{n1}   \\
    \vdots & \ddots & \vdots \\
    \hat{f}^{(k)}_{1n} & \cdots & \hat{f}^{(k)}_{nn} 
    \end{bmatrix} .
    \end{array}
  \end{equation}
Indeed, the $i,j$ block of the matrix $X^{(k)}$, which we denote by  $X^{(k)}_{ij}$,  is
\begin{equation}
  X^{(k)}_{ij} = \rho_k (g_i) \rho_k^{*}(g_j) = \rho_k(g_i g_j^{-1}) .
\end{equation}
Therefore, recovering the matrices $X^{(k)}_{ij}$ which take the above form is equivalent to recovering the ratio $g_i g_j^{-1}$ between pairs, which allows us to recover $g_1,\ldots ,g_n$ up to a global element. 
In other words, we have ``lifted'' the problem from the original variables $g_1,\ldots ,g_n$ to the block matrices,
where each block is associated with the ratio $g_i g_j^{-1}$ between a pair.

\subsubsection{Convex Relaxation of NUG}

We would like to convexify the NUG problem in order to use convex optimization theory and algorithms;
in this section we consider the convex relaxation of (\ref{eq:NUG:Fourier:Block}) and (\ref{eq:NUG:Fourier:Block:structure1}):
\begin{equation}\label{eq:NUG:convex}
\underset{ \{X^{(k)}\}_{k=0}^{\infty} }{\argmin}  \sum_{k=0}^\infty tr \left( \hat{F}^{(k)} X^{(k)} \right) 
\end{equation}
where the solution matrices $X^{(0)}, X^{(1)},\ldots $ are in the convex hull of 
the matrices defined in (\ref{eq:NUG:Fourier:Block:structure1}).

%
%
%
The following {\bf SDP relaxation}  has been proposed in \cite{bandeira2015non}:
  \begin{equation}\label{eq:NUG:convexrelax}
    \begin{array}{lll}
      \underset{X^{(0)},X^{(1)},\ldots }{\argmin} & \sum_{k=0}^\infty tr \left( \hat{F}^{(k)} X^{(k)} \right) & \\
      {\text{subject to}}  & X^{(k)} \succeq 0 & \forall k \\
                     & X_{ii}^{(k)} = I_{d_k \times d_k} & \forall k,i \\
                     & \sum_{k=0}^\infty d_k tr \left( \rho^*_k(g) X^{(k)}_{ij}  \right) \ge 0 & \forall 1 \le i,j \le n ~,~ \forall g\in\mathcal{G} \\
                     & X_{ij}^{(0)} = 1  & \forall 1\le i,j \le n \\
  \end{array}
  \end{equation}
Where,  
  \begin{equation}
    X^{(k)} =  \begin{bmatrix}  X_{11} & \cdots & X_{1n}   \\
    \vdots & \ddots & \vdots \\
    X_{n1} & \cdots & X_{nn} 
    \end{bmatrix} .
  \end{equation}
The constraints in (\ref{eq:NUG:convexrelax}) are designed to restrict $X^{(k)}$ in (\ref{eq:NUG:convexrelax}) to the convex hull of the matrices in (\ref{eq:NUG:Fourier:Block:structure1}).

\begin{remark}
When the expansion of the irreducible representations on $\mathcal{G}$ is infinite,
it must be truncated in practice.
The implementation of the non-negativity constraint 
$\sum_{k} d_k tr \left(  \rho_{(k)}^*\left(g\right) X^{(k)}_{ij}  \right) \ge 0 $
is not trivial.
The problem is discussed in \cite{bandeira2015non}, where  
$\mathcal{G}$ is sampled  and a non-negative kernel is applied. 
In some cases, Sum-of-Squares (SOS) constraints can also be used. 
The constraint, and possible improvements of it, are the subject of ongoing work.
\end{remark}

%
\section{NUG Formulation for Simultaneous Classification and Alignment}\label{sec:analysis}

\subsection{Motivating Example: Classification and Alignment over SO(2)}\label{sec:ex}

In this section we present the problem of multireference alignment on SO(2),
and a heterogeneity problem associated with it. 
This problem turns out to be simpler than the Cryo-EM problem in some fundamental ways which we will discuss in Section \ref{sec:res}, in the sense that there are tools available for approaching this problem that are not available in Cryo-EM; however, in the context of the NUG formulation, the problem has many of the features of the Cryo-EM problem.

Suppose that we have some periodic function $\psi : [0,2\pi) \rightarrow \mathbb{C}$ over SO(2), 
and suppose that we are given multiple copies of this function,
each shifted by some arbitrary angle. An example of such shifted copies is given in Figure \ref{fig:sig1:so2}.
If we want to recover the original function (up to cyclic shifts), we may choose an arbitrary copy, because all the copies are identical to the original function up to shifts.

Next, suppose that we have noisy shifted copies of the function (Figure \ref{fig:sig1:so2:noise}(a)).
If we wish to approximate the original function (up to shifts), we would align the noisy copies
(Figure \ref{fig:sig1:so2:noise}(b)) and then average them to cancel out the noise (Figure \ref{fig:sig1:so2:noise}(c)).
Of course, in order to do this we must somehow recover the correct shifts of all the copies together (up to some global shift). 
In the following sections, we will use a penalty function for different possible pairwise alignment;
for each pair of copies, we can define a ``compatibility penalty'' for different possible alignments,
for example (with slight abuse of notation), via the formula
\begin{equation}\label{eq:ex:fij}
  f_{i,j} \left( g \right) = \| \varphi_i - g \circ \varphi_j \|_2^2 = \frac{1}{2\pi} \int_0^{2\pi} | \varphi_i(\theta)-  \varphi_j(\theta- g) |^2 d\theta .
\end{equation}
An example of such compatibility penalty function is given in Figure \ref{fig:fij}.
When the shifts are unknown, the problem of aligning the signals is a NUG (see \cite{bandeira2015non,bandeira2014multireference}).

In the heterogeneity problem we have a mixture of prototype signals;
in this simplified example, let us assume that we have a mixture of noisy shifted versions
of two classes of functions $\psi_1$ and $\psi_2$, so that each sample is a shifted noisy version of either $\psi_1$ or $\psi_2$ as illustrated in the example in Figure \ref{fig:het}(a). 
If we knew both the class and shift of each sample, we could divide the samples into two classes,
and align them within each class (Figure \ref{fig:het}(b),(c)),
so that we could average within each class and approximate the two original signals (Figure \ref{fig:het}(d),(e)).

We know neither the shift nor the class of the samples;
we study the extension of NUG to this case of alignment in the presence of heterogeneity.

\begin{figure}[h!]
    \centering
    {\includegraphics[height=45mm,angle=0]{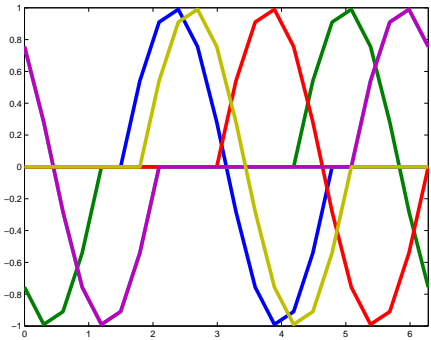}}
    \caption{Shifted copies of a function on SO(2)}\label{fig:sig1:so2}
  \end{figure}

  \begin{figure}[h!]
   \begin{subfloat}[Shifted noisy copies]
    {\includegraphics[width=40mm,angle=0]{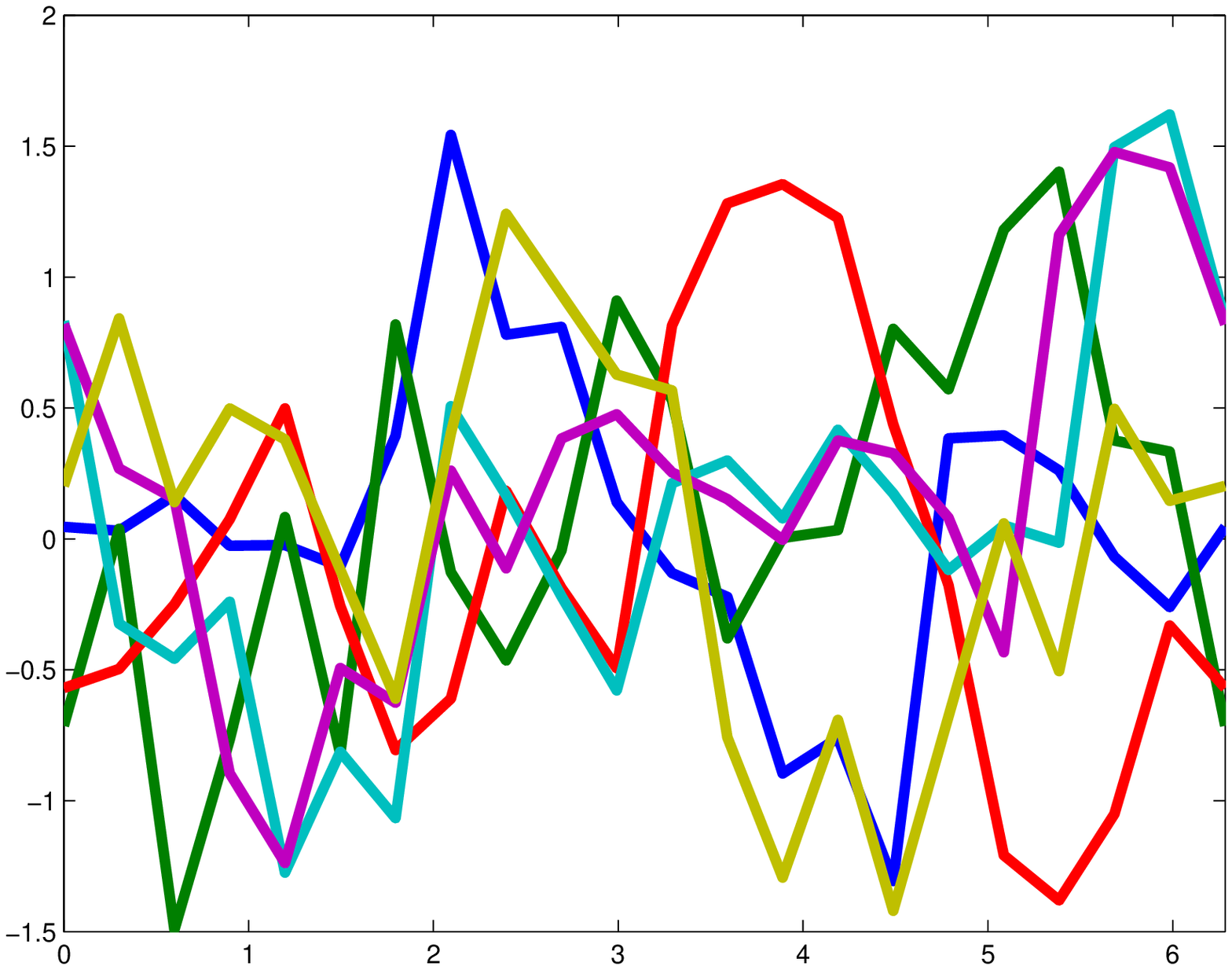}}
  \end{subfloat}	
  \begin{subfloat}[Aligned noisy copies]	
    {\includegraphics[width=40mm,angle=0]{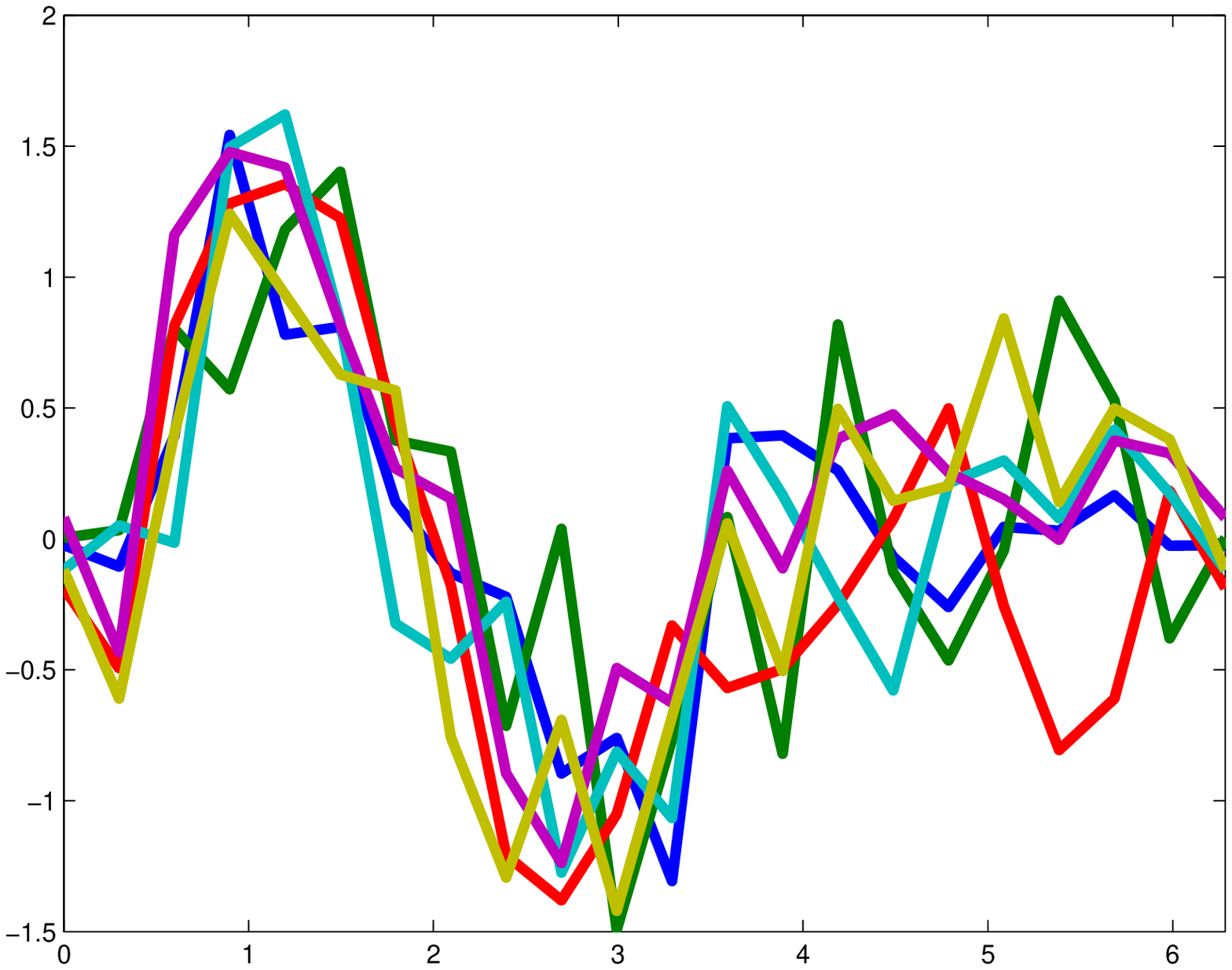}}
  \end{subfloat}	
  \begin{subfloat}[Averaged aligned copies (blue) vs. original function (red)]	
    {\includegraphics[width=40mm,angle=0]{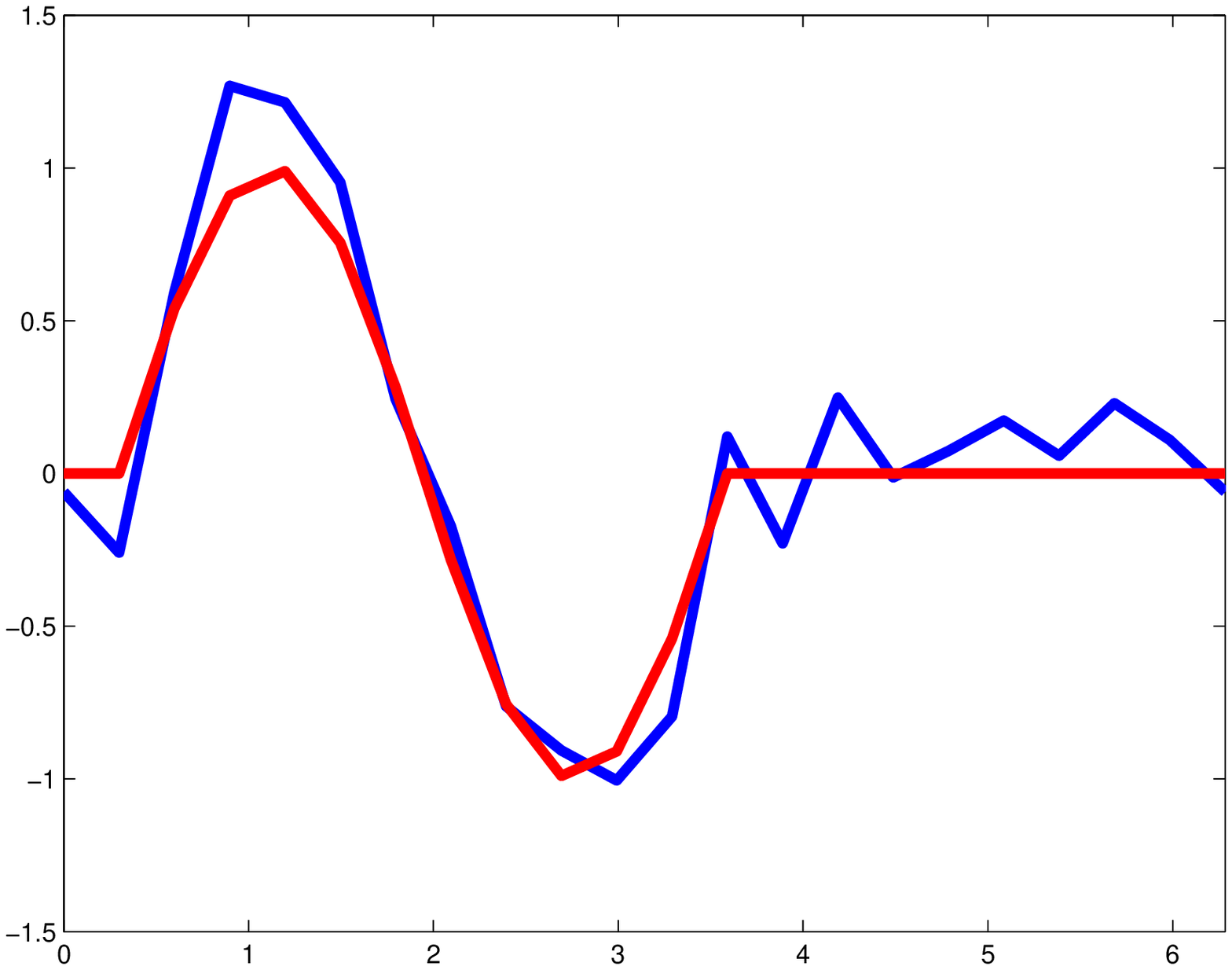}}
  \end{subfloat}	
  \caption{Noisy shifted copies of a function on SO(2)}\label{fig:sig1:so2:noise}
  \end{figure}

  \begin{figure}[h!]
    \centering
    \begin{subfloat}[Two shifted copies]	
    {\includegraphics[width=50mm,angle=0]{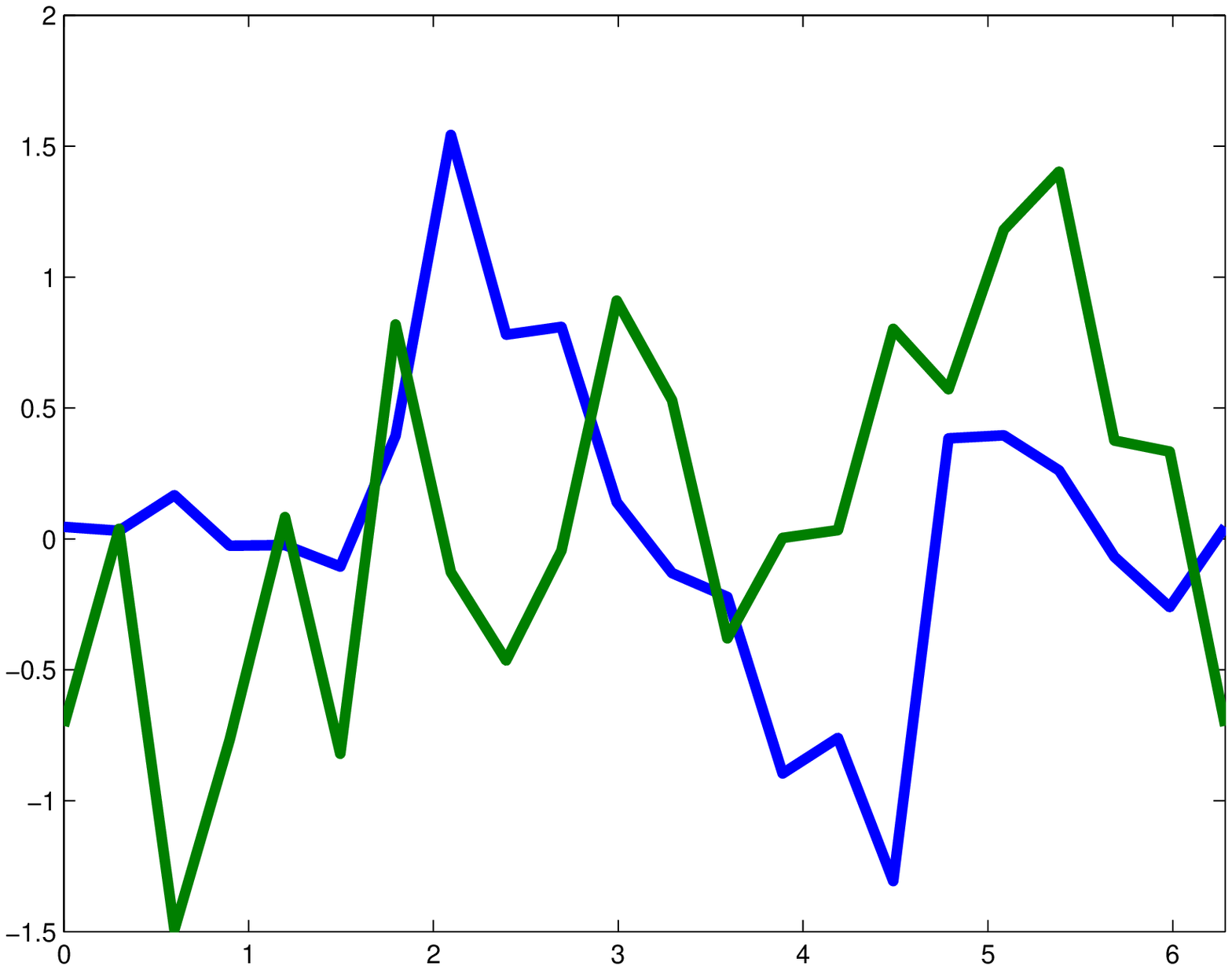}}
    \end{subfloat}	
    \begin{subfloat}[``Alignment penalty'' $f_{ij}(g)$]
    {\includegraphics[width=50mm,angle=0]{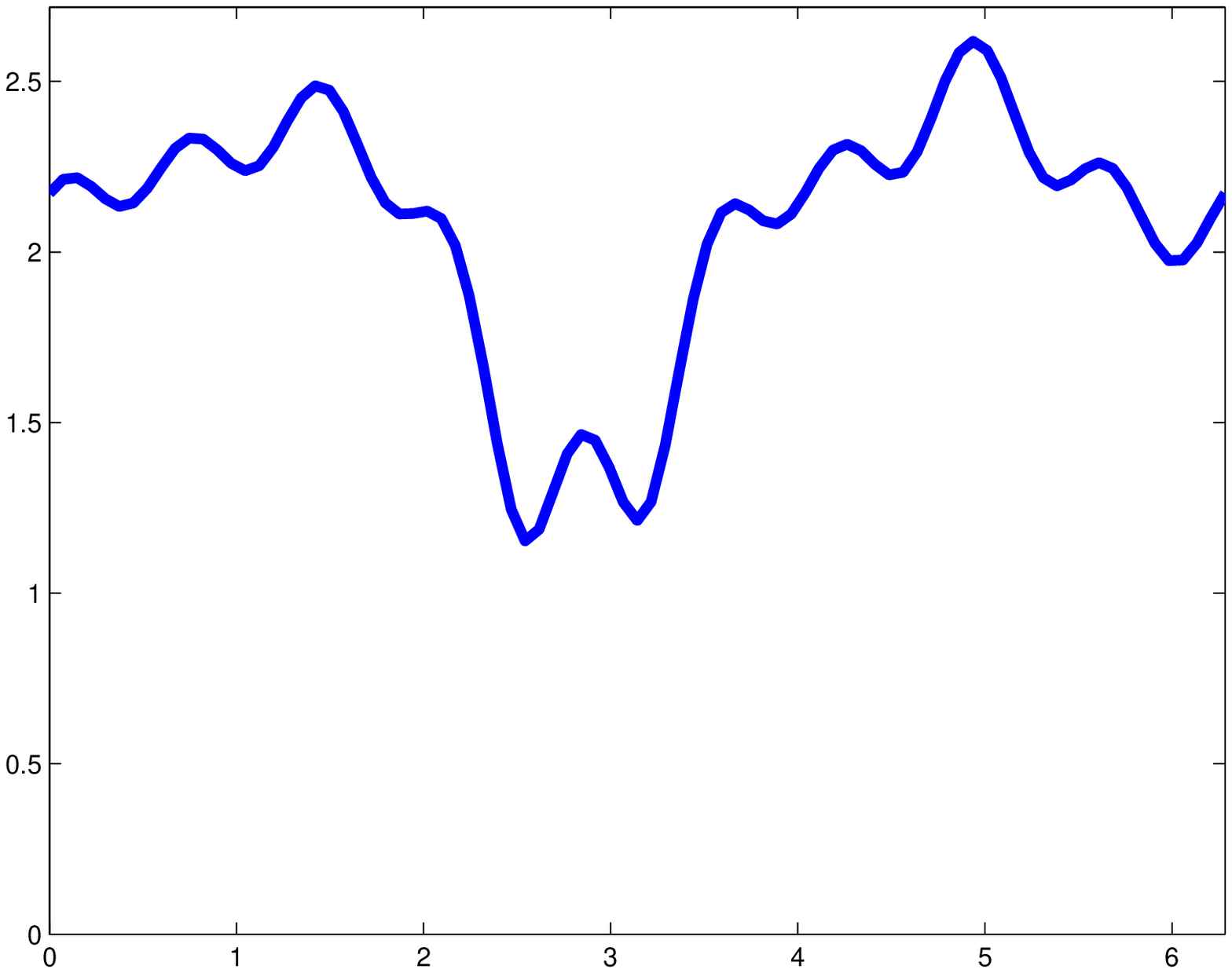}}
    \end{subfloat}	
    \caption{Penalty function for alignment of signals}\label{fig:fij}
  \end{figure}

\begin{figure}[h!]
  \begin{tabular}{ccc}
    \multirow{5}{*}{\includegraphics[width=40mm,angle=0]{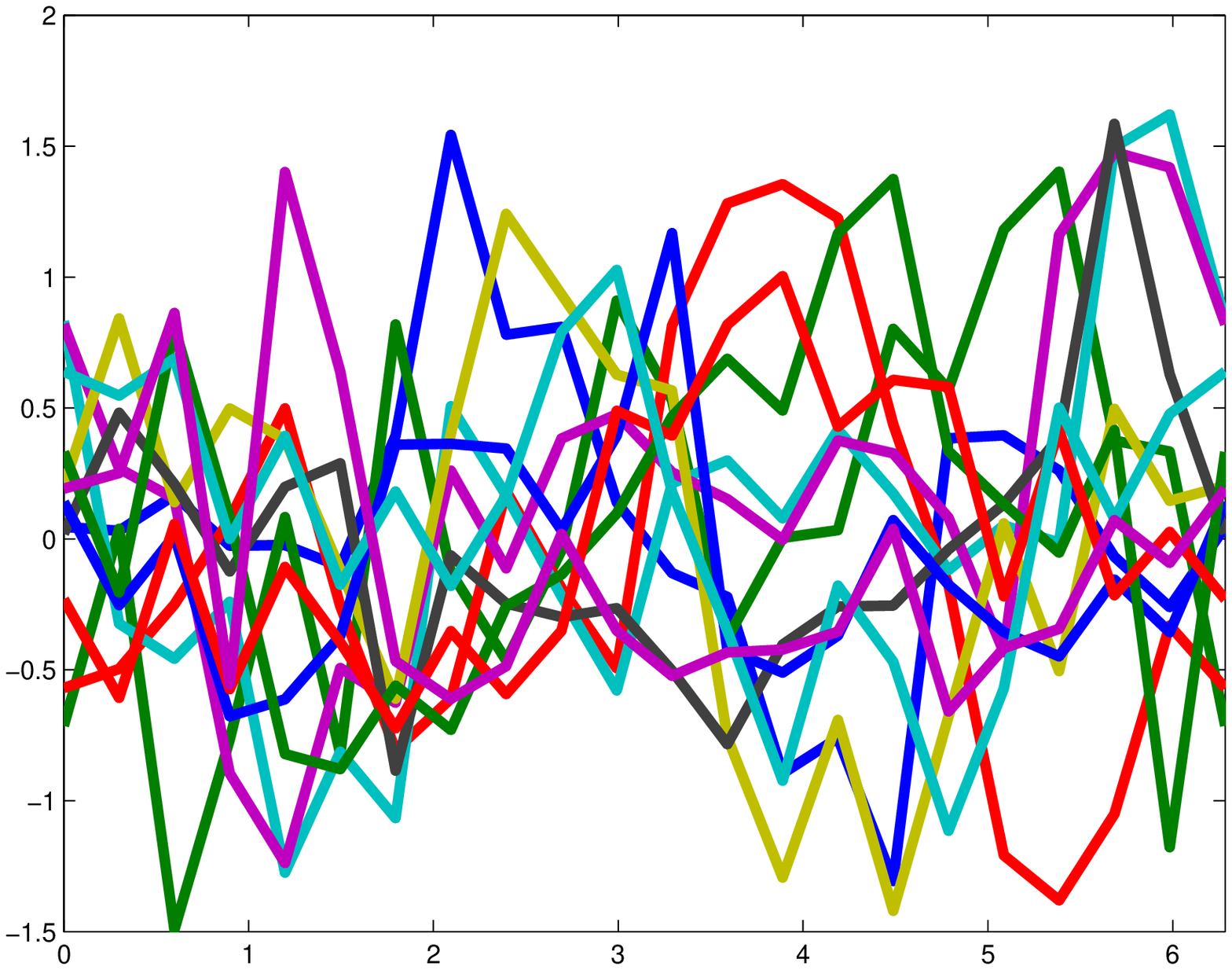} }   & 
    {\includegraphics[width=35mm,angle=0]{figures/misc_signals_v006_sig1_back_ns}} & {\includegraphics[width=35mm,angle=0]{figures/misc_signals_v006_sig1_back_avg_ns}} \\
       & (b) copies of class 1,  & (d) averaged class 1  \\
      &  aligned &  vs. original $\psi_1$  \\
      & {\includegraphics[width=35mm,angle=0]{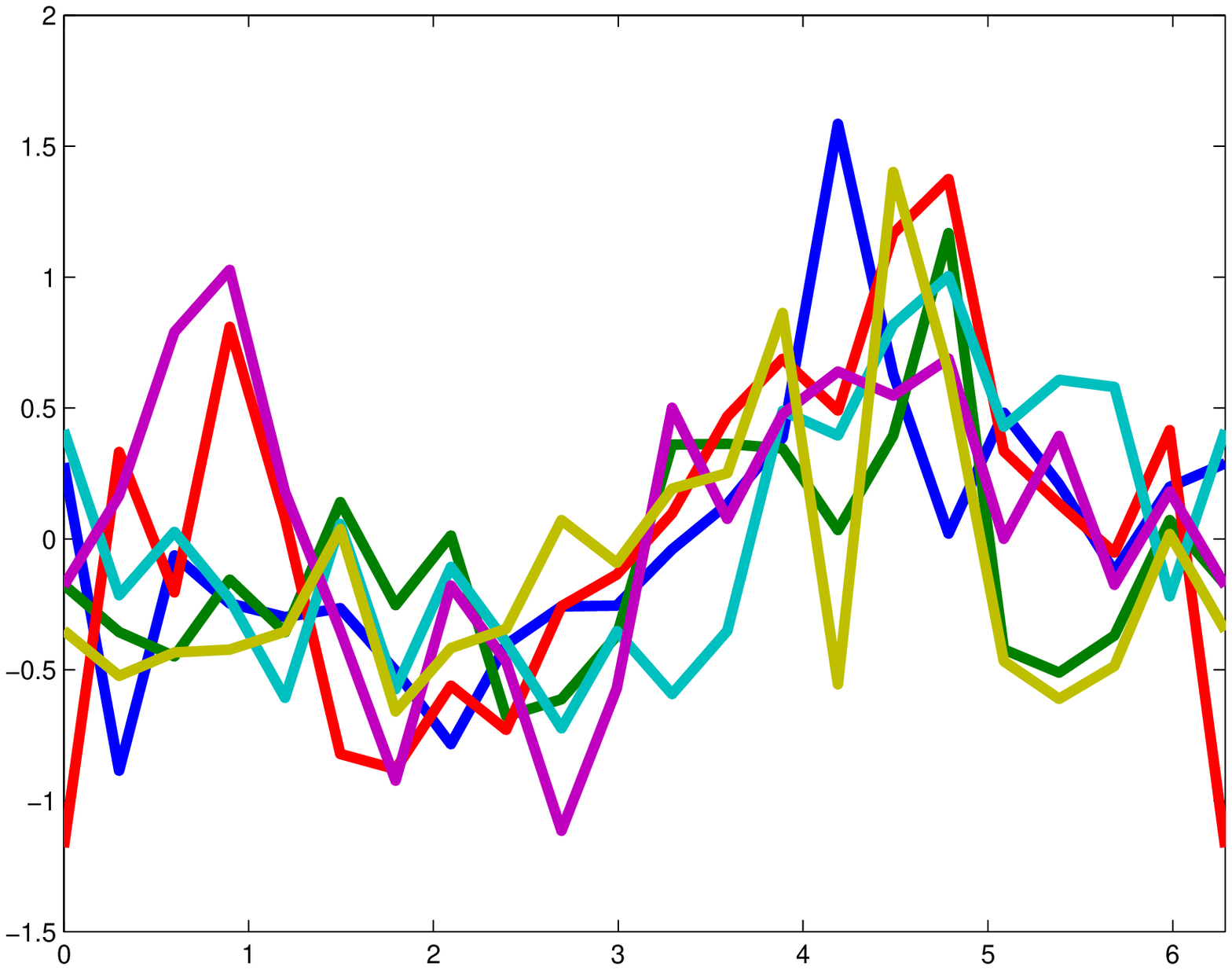}} & {\includegraphics[width=35mm,angle=0]{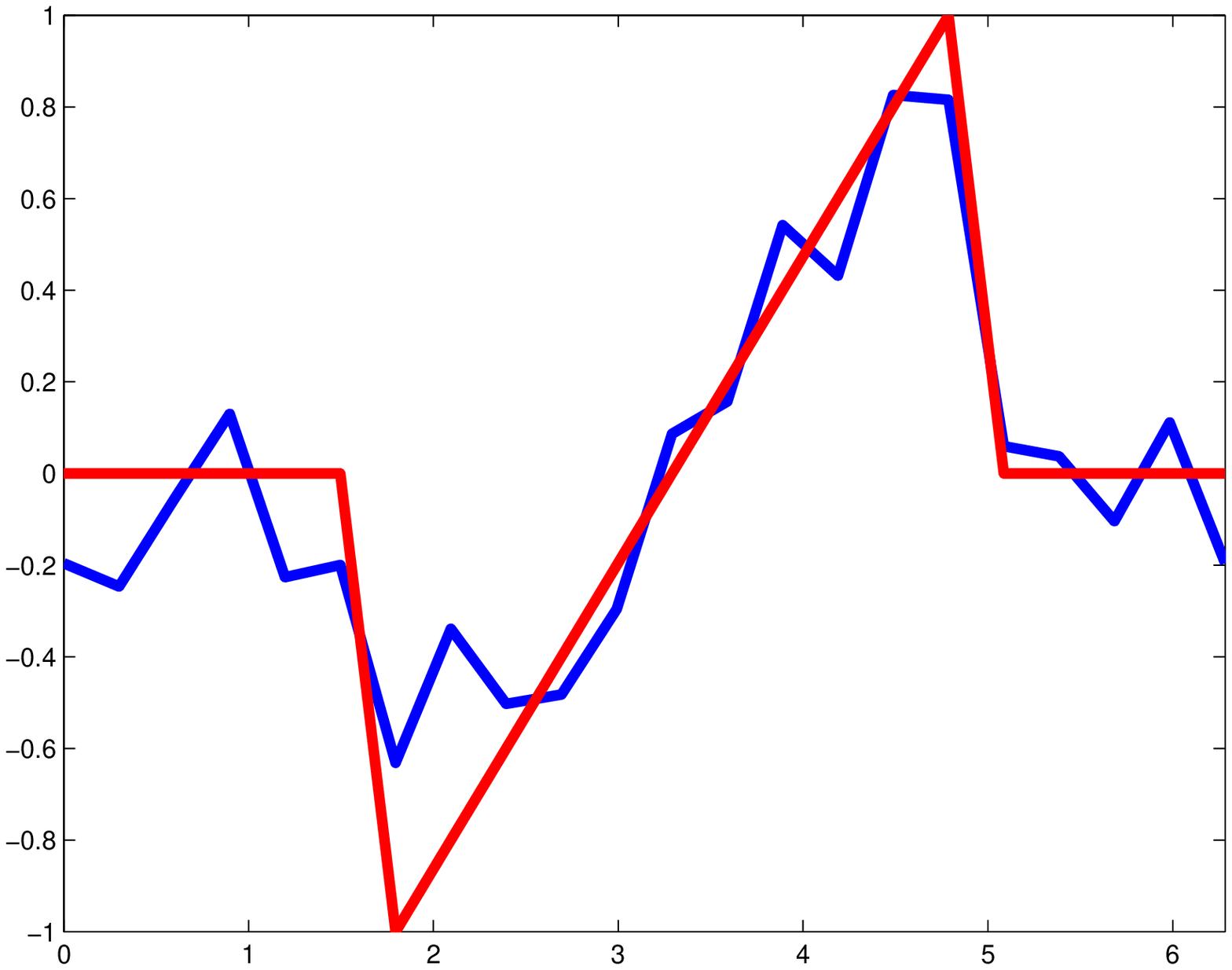}} \\
      (a) mixture of signals & (c) copies of class 2,  & (e) averaged class 2  \\
      &  aligned &  vs. original $\psi_2$  \\
  \end{tabular}
  \caption{Classification and alignment over SO(2)}\label{fig:het}
\end{figure}

\clearpage

\subsection{Problem Formulation}

We would like to find the optimal way to divide the samples into $M$ classes,
so that we can best align them within each class.
More formally, we would like to optimize the rotations and classification together:
\begin{equation}\label{eq:prob:aligncluster}
  \underset{\begin{subarray}{c} g_1,\ldots ,g_n \in \mathcal{G} \\
  a_1,\ldots ,a_n \in {0,..,M-1}  \end{subarray} }{\argmin} 
\sum_{m=0}^{M-1} \sum_{\begin{subarray}{c} i,j: \\ a_i=m \\ a_j=m \end{subarray}} f_{ij}(g_i  g_j^{-1}) .
\end{equation}

\begin{remark}
In this formulation, it is typically assumed  that the penalty $f_{ij}$ is non-negative, and typically larger when $i$ and $j$ do not belong to the same class, so that there is an incentive to distribute the samples among $M$ clusters, and align them within each cluster.
\end{remark}

We will also discuss the problem of controlling the distribution to different clusters; 
for example, we will discuss the case where all the clusters are required to be of equal size:
\begin{equation}
|\{i:a_i=m\}|=n/M .
\end{equation}

%
%
%
\subsection{Ambiguity and Averaging}\label{sec:avg}

In some cases, there is a degree of ambiguity in a solution of a NUG (in addition to the
inherent global ambiguity discussed in Remark \ref{rem:unique}).
Suppose that $g_1,g_2,\ldots,g_n$ is a solution of the NUG in (\ref{eq:NUG:Fourier:Block})
with the corresponding matrices $X^{(0)}, X^{(1)}, \ldots $,
and suppose that there exists another solution $\tilde{g}_1,\tilde{g}_2,\ldots,\tilde{g}_n$ 
with corresponding matrices $\tilde{X}^{(0)}, \tilde{X}^{(1)}, \ldots$
that achieves the same optimization objective. 
We would be particularly interested in the case where $\tilde{g}_1,\tilde{g}_2, \ldots, \tilde{g}_n$ cannot be obtained by applying some group element to $g_1,g_2,\ldots, g_n$ (the case discussed in 
Remark \ref{rem:unique}), so that in general ${X}^{(k)} \ne \tilde{X}^{(k)}$.
In the convex formulation of the problem in (\ref{eq:NUG:convex}), if both 
$X^{(0)}, X^{(1)}, \ldots$ and $\tilde{X}^{(0)}, \tilde{X}^{(1)}, \ldots $ are solutions, 
then so is every convex combination $\overline{X}^{(0)}, \overline{X}^{(1)}, \ldots$ of those solutions, even if there is no ``physical'' solution $\overline{g}_1,\overline{g}_2,\ldots , \overline{g}_n$ which corresponds to $\overline{X}^{(0)}, \overline{X}^{(1)}, \ldots$.
In some cases, where the form of the ambiguity is known, we can use this property to enforce a solution of a certain form. 
An example is provided in the next section.


%
%
%
\subsection{Reducing k-clustering to a NUG}\label{sec:cluster}

In this section we discuss the NUG formulation of the problem of clustering vertices in a graph in k communities,
to which we refer as k-clustering or k-classification. 
In particular, we  discuss the max-k-cut problem and the balanced version of the problem (where each cluster contains an equal number of vertices). 
The SDP relaxation of max-k-cut has been studied in \cite{goemans1995improved,frieze1995improved} and the closely related  min-k-cut problem has been studied as a NUG in \cite{bandeira2015non}. 
We present a slightly different formulation and derivation which we find useful for our discussion.
Since  ``$k$'' is often reserved for denoting indices of irreducible representations, we denote the number of clusters by $M$.

Given an undirected weighted graph $(V,E)$, the max-k-cut problem is to divide the vertices of a graph into $M$ clusters, cutting the most edges between clusters
  \begin{equation}
    \underset{a_1,\ldots ,a_n \in 0,\ldots ,M-1}{\argmax} \sum_{i,j=1}^n \left(1-\delta(a_i-a_j)\right)w_{ij} ,
  \end{equation}
with $w_{ij}$ the weight of the edge between vertices $i$ and $j$.
In other words, the problem is to divide the graph into $M$ clusters retaining the minimal sum of edge weights:
  \begin{equation}\label{eq:kmaxcut:def}
    \underset{a_1,\ldots,a_n \in 0,\ldots,M-1}{\argmin} \sum_{i,j=1}^n f_{ij}(a_i-a_j) ,
  \end{equation}
where $f_{ij}(a) = w_{ij} \delta(a)$.
We can view the weight of each edge as a measure of {\em incompatibility} or ``distance,'' and attempt to classify the vertices into clusters which are the least incompatible;
i.e.  the goal is to minimize the sum of intra-cluster weights retained, by finding a clustering that removes as many inter-cluster edges as possible.

The following SDP relaxation has been proposed in \cite{goemans1995improved,frieze1995improved},
\begin{equation}\label{eq:kcluster1}
\begin{array}{lll}
\underset{Y}{\text{min}} &  tr \left( W  Y \right)  & \\
{\text{Subject to}}  & Y \succeq 0 &  \\
                     & Y_{ii} = 1 & \forall i \\
                     & Y_{ij} \ge -\frac{1}{M-1}  & \forall i,j 
\end{array}
\end{equation}
where $W$ is the matrix of edge weights.
In a solution that corresponds to a ``physical'' solution (a valid classification, rather than, for example, a convex combination of classifications), $Y_{ij}=1$ if $i$ and $j$ are in the same cluster, and  $Y_{ij} = -\frac{1}{M-1}$ otherwise. 
A derivation for the related min-k-cut problem, in the context of NUG, is provided in \cite{bandeira2015non}. 
We discuss an additional derivation which we will generalize in the following sections. 

We consider the group $\mathbb{Z}_M$ of cyclic shifts. 
A function on this group can be written explicitly as a vector of length $M$, indexed $0,1,\ldots,M-1$.
We define the function $f_{ij}$ by the following formula   
\begin{equation}
  f_{ij} = (w_{ij},0,0,\ldots)^\intercal ,
\end{equation}
where $w_{ij}$ is the weight of the edge between $i$ and $j$.
We denote by $a_i$ the class assignment of the $i$ element, so that 
\begin{equation}
  f_{ij}(a_i-a_j) = \begin{cases} \begin{array}{lll} 0 &:& a_i \ne a_j \\
 w_{ij} &:& a_i=a_j \end{array} \end{cases},
\end{equation}
or, in group notation
\begin{equation}
  f_{ij}(a_i a_j^{-1}) = \begin{cases} \begin{array}{lll} 0 &:& a_i a_j^{-1} \ne \mathrm{e} \\ w_{ij} &:&  a_i a_j^{-1} = \mathrm{e} \end{array} \end{cases},
\end{equation}
where $\mathrm{e}$ is the identity element. This $f_{ij}$ is precisely the penalty function $f_{ij}$ in (\ref{eq:kmaxcut:def}).

The discrete Fourier transform (DFT) of $f_{ij}$ (with the appropriate choice of norm) is 
\begin{equation}\label{eq:cluster:NUG:costfunc}
  \hat{f}_{ij} = \frac{1}{M}(w_{ij},w_{ij},w_{ij},\ldots)^\intercal .
\end{equation}
These coefficients coincide with the coefficients of the expansion of $f_{ij}$ in the irreducible representation of $\mathbb{Z}_M$:
\begin{equation}\label{eq:cluster:NUG:costfunc:exp}
  {f}_{ij}(a) = \sum_{m=0}^{M-1} \hat{f}_{ij}(m) e^{\mathrm{i} 2 a \pi m/M}.
\end{equation}
Rewriting the clustering problem  (\ref{eq:kmaxcut:def}) as a NUG over $\mathbb{Z}_M$ yields  
\begin{equation}\label{eq:cluster:NUG:zm}
\underset{a_1,\ldots,a_n \in \mathbb{Z}_M }{\argmin} \sum_{i,j=1}^n f_{ij} \left( a_i  a_j^{-1} \right) ,
\end{equation}
and substituting (\ref{eq:cluster:NUG:costfunc}) and (\ref{eq:cluster:NUG:zm}) 
into the block matrix formulation in (\ref{eq:NUG:Fourier:Block}) yields
\begin{equation}\label{eq:cluster:NUG:Block}
\begin{array}{lll}
\underset{X^{(0)},\ldots,X^{(M-1)}}{\argmin} & \sum_{m=0}^{M-1} tr \left( \hat{F}^{(m)} X^{(m)} \right) & \\
\end{array}
\end{equation}
subject to $X^{(m)}$ having the structure in (\ref{eq:NUG:Fourier:Block:structure1}).
The scalar irreducible representations here are $\eta_m(a)=e^{\mathrm{i} 2 \pi a m/k }$,
so that for every $m=0,1,\ldots,M-1$, the matrix $X^{(m)}$ is an $n \times n$ matrix with $X^{(m)}_{ij}$ in position $i,j$. 
The matrix $\hat{F}^{(m)}$ is a matrix of the coefficients $\hat{f}_{ij}(m)$ in the DFT of $f_{ij}$;
by (\ref{eq:cluster:NUG:costfunc}), $\hat{f}^{(m)}_{ij} = w_{ij}/M$, for all $m$. 
For some solution of the NUG, we have for every pair $i,j$, with $a_i  a_j^{-1} = a_{ij} $
\begin{equation}\label{eq:cluster:NUG:Xij}
   X^{(m)}_{ij} = e^{\mathrm{i} 2 \pi a_{ij} m/M }, 
\end{equation}
where we again use $a_{ij}$ as the group elements and the angle.

After writing the problem in the block matrix form, we turn our attention to 
the convex version of this formulation (see (\ref{eq:NUG:convex})). 
In particular, we discuss the ambiguity in the solution, which results in convex combinations of equivalent solutions, as discussed in Section \ref{sec:avg}.
The penalty function $f_{ij}( a_i  a_j^{-1})$ depends only on whether or not $i$ and $j$ are in the same class, so it is invariant to permutations. In other words, for any permutation $\sigma$,
\begin{equation}
\sum_{i,j=1}^n f_{ij} \left( a_i  a_j^{-1} \right) = \sum_{i,j=1}^n f_{ij} \left( \sigma(a_i)  \left(\sigma(a_j)\right)^{-1}\right) .
\end{equation}

It follows that we can average all the different permutations, as discussed in Section \ref{sec:avg}.
If $i$ and $j$ are assigned to the same class in the solution, then $a_i=a_j$ so $a_{ij}=a_i a_j^{-1}=\mathrm{e}$ (or in integer notation $a_i-a_j = 0$) and by (\ref{eq:cluster:NUG:Xij})
\begin{equation}
  X^{(0)}_{ij} = X^{(1)}_{ij} = \ldots = X^{(M-1)}_{ij} = 1.
\end{equation} 
However, if $i$ and $j$ are not assigned to the same class in the solution, we can average all the solutions for all permutations $\sigma$
\begin{equation}
   X^{(m)}_{ij} = e^{\mathrm{i} 2 \pi (\sigma(a_i)(\sigma(a_j))^{-1} )m/M } = e^{\mathrm{i} 2 \pi (\sigma(a_i)-\sigma(a_j) )m/M } .
\end{equation}
A simple computation yields the averaged (equally weighted convex combination) solution for all $m>0$:
\begin{equation}
   \overline{X}^{(m)}_{ij} = \frac{1}{M-1} \sum_{a=1}^{M-1} e^{\mathrm{i} 2 \pi a m/M } = -\frac{1}{M-1}.
\end{equation}
In other words, $X^{(0)}$ is the all ones matrix, 
and the matrices for all $m>0$  are equal:
\begin{equation}\label{eq:cluster:averagedx:equal}
  X^{(1)} = X^{(2)} = \ldots  = X^{(M-1)},
\end{equation}
with the element $X^{(m)}_{ij}$ of these matrices with $m>0$:
\begin{equation}
  X^{(m)}_{ij} = \begin{cases} \begin{array}{lll} 1 &,& \text{ if $i$ and $j$ are in the same class,} \\ -\frac{1}{M-1} &,& \text{otherwise}.\end{array} \end{cases}
\end{equation}
Since $X^{(0)}_{ij} = 1$ is fixed, it can be ignored in the penalty term of (\ref{eq:cluster:NUG:Block}), so the optimization is reduced to 
\begin{equation}
\underset{X^{(1)},\ldots ,X^{(M-1)}}{\argmin}  \sum_{m=1}^{M-1} tr \left( \hat{F}^{(m)} X^{(m)} \right). 
\end{equation}
Using (\ref{eq:cluster:averagedx:equal}) and (\ref{eq:cluster:NUG:costfunc}), the optimization is further reduced to 
\begin{equation}
 \underset{X^{(1)}}{\argmin}  \left((M-1) tr \left( \hat{F}^{(1)} X^{(1)} \right)\right),
\end{equation}
which is scaled to 
\begin{equation}
\underset{X^{(1)}}{\argmin}  \left(tr \left( \hat{F}^{(1)} X^{(1)} \right) \right).
\end{equation}
Setting $X^{(1)}=Y$, we have the optimization term in (\ref{eq:kcluster1}), 
with the other conditions in (\ref{eq:kcluster1}) following from the derivation above.

\subsection{Controlling Cluster Size or Distributions}\label{sec:sizecontrol}

The purpose of this section is to extend the NUG framework by adding constraints on the distribution of solutions over the group.

In some cases it is useful to restrict the clusters in a graph cut problem to
be of equal size (for example, see discussion of min-k-cut in \cite{agarwal2015multisection}), i.e.
\begin{equation}
|\{i:a_i=m\}|=n/M .
\end{equation}
The NUG formulation does not have a mechanism to enforce such a constraint.
We first consider the extension of the NUG in (\ref{eq:kcluster1}) for the max-k-cut problem to the case of balanced cluster size. 
We add the constraint that for $m>0$, 
\begin{equation}\label{eq:equidistrib:sum0}
   \sum_{j} X^{(m)}_{ij} = 0   ~~~\forall i  
\end{equation}
(for $m=0$, the matrix $X^{(0)}$ is the trivial all ones matrix).
Indeed, for any valid balanced solution, every vertex $i$ has $n/M$ vertices (including itself) in the same cluster, and for these vertices $ X^{(m)}_{ij}=1$; 
every vertex also has $\frac{n}{M}(M-1)$ vertices in different classes, for these vertices $ X^{(m)}_{ij}=-\frac{1}{M-1}$. Therefore, the sum of these elements is $0$. 
This solution resembles  the algorithm proposed in \cite{agarwal2015multisection}.

This idea is a special case of a more general framework that enforces constant distribution over the group by enforcing (\ref{eq:equidistrib:sum0}).
The strict constraint on the distribution can be relaxed to an approximation, 
and therefore extended beyond discrete groups by relaxing the condition to one of the following constraints
\begin{equation}
   \| \sum_{j} X^{(m)}_{ij}(q) \|^2 \le w(m)   ~~~\forall i ,
\end{equation}
\begin{equation}
   \| \sum_{ij} X^{(m)}_{ij}(q) \|^2 \le w(m)  ,
\end{equation}
\begin{equation}
   \sum_i \| \sum_{j} X^{(m)}_{ij}(q) \|^2 \le w(m)  ,
\end{equation}
or by adding a similar constraint as a regularizer in the optimization (with the obvious extension where the irreducible representation $X^{(m)}_{ij}$ is a matrix).
This approach, which views the irreducible representations and their sum as an approximation of the Haar measure of the group (or appropriate variation when a prior is available), will be discussed in more detail in a future paper.

\subsection{The Direct Product of Alignment and Classification (Product NUG)}\label{sec:nugprod}

We revisit (\ref{eq:prob:aligncluster}) and rewrite the summation in the optimization:
\begin{equation}\label{eq:directprod:aligncluster:sum1}
\sum_{m=0}^{M-1} \sum_{\begin{subarray}{c} i,j: \\ a_i=m \\ a_j=m \end{subarray}} f_{ij}(g_i  g_j^{-1}) = 
\sum_{i,j=1}^n \delta(a_i, a_j) f_{ij}(g_i g_j^{-1}),
\end{equation}
where 
\begin{equation}\label{eq:delta1}
  \delta(a_i,a_j) = \begin{cases} \begin{array}{lll} 1 &:& a_i = a_j \\ 0 &:& \text{otherwise} \end{array} \end{cases}.
\end{equation}
With a small abuse of notation, we rewrite the class labels $a_1,\ldots ,a_n$ as elements in $\mathbb{Z}_M$; the expression $a_i = a_j$ can also be written as $a_i a_j^{-1} = e$ (where $e$ is the identity element of $\mathbb{Z}_M$), so, we can also write (\ref{eq:delta1}) as:
\begin{equation}\label{eq:delta2}
   \delta(a_i,a_j) = \delta(a_i a_j^{-1}) = \begin{cases} \begin{array}{lll} 1 &:& a_i  a_j^{-1} = e \\ 0 &:& \text{otherwise} . \end{array} \end{cases}
\end{equation}
We introduce the function $\tilde{f}_{ij}: \mathcal{G} \times \mathbb{Z}_M \rightarrow \mathbb{R}$, defined as
\begin{equation}
  \tilde{f}_{ij}\left( \left( g , a \right) \right) = 
 {f}_{ij}\left(  g \right) \delta(a) .
\end{equation}
Using the identity (\ref{eq:groupprod:inv}), we obtain
\begin{equation}
\tilde{f}_{ij}\left( (g_i,a_i)  (g_j,a_j)^{-1} \right) = \tilde{f}_{ij}\left( \left( g_i g_j^{-1} , a_i  a_j^{-1} \right) \right) ,
\end{equation}
and observe that 
$\tilde{f}_{ij}\left( (g_i,a_i) (g_j,a_i)^{-1} \right)$ is now simply a function over the compact group $\mathcal{G} \times \mathbb{Z}_M$. Therefore, the expression in (\ref{eq:prob:aligncluster}) is reduced to the NUG
\begin{equation}
  \underset{\begin{subarray}{c} (g_1,a_1),\ldots,(g_n,a_n) \in \mathcal{G} \times \mathbb{Z}_M  
\end{subarray} }{\argmin} 
\sum_{i,j} \tilde{f}_{ij}\left( (g_i,a_i) (g_j,a_j)^{-1} \right)  .
\end{equation}

The block matrix formulation (\ref{eq:NUG:Fourier:Block}) of this product NUG is
  \begin{equation}\label{eq:directprod:aligncluster:Fourier:Block}
  \underset{(g_1,a_1),\ldots ,(g_n,a_n) \in \mathcal{G} \times \mathbb{Z}_M }{\argmin} \sum_{k=0}^\infty\sum_{m=0}^{M-1} tr \left( \hat{F}^{(k,m)} X^{(k,m))} \right) 
  \end{equation}
  Where,  
  \begin{equation}\label{eq:directprod:aligncluster:Fourier:Block:structure1}
    \begin{array}{c}
    X^{(k,m)} =  \begin{bmatrix} \psi_{k,m} \left( (g_1,a_1) \right) \\ \vdots \\ \psi_{k,m} \left( (g_n,a_n) \right) \end{bmatrix} \begin{bmatrix} \psi_{k,m} \left( (g_1,a_1) \right) \\ \vdots \\ \psi_{k,m} \left( (g_n,a_n) \right) \end{bmatrix}^* , 
    \\
    \hat{F}^{(k,m)}= d_{km} \begin{bmatrix}  \hat{f}_{11}(k,m) & \cdots & \hat{f}_{n1}(k,m)   \\
    \vdots & \ddots & \vdots \\
    \hat{f}_{1n}(k,m) & \cdots & \hat{f}_{nn}(k,m)
    \end{bmatrix} ,
    \end{array}
  \end{equation}
with $\hat{f}_{ij}(k,m)$ the Fourier coefficient of $\tilde{f}_{ij}$ corresponding to the irreducible representation $\psi_{k,m}$,  and $d_{km}$ the dimensionality of that irreducible representation.
The irreducible representations $\psi_{k,m}$ of $\mathcal{G} \times \mathbb{Z}_M$ are enumerated in Table \ref{tb:directprod:aligncluster:ireps}; they are referenced by two indices, $k=0,1,\ldots$ and 
$m=0,1,\ldots,M-1$.

As in the general discussion of NUG, we are interested in the convex relaxation of (\ref{eq:directprod:aligncluster:Fourier:Block}):
\begin{equation}\label{eq:directprod:aligncluster:convex}
\underset{ \{X^{(k,m)}\}_{k,m} }{\argmin}  \sum_{m=0}^{M-1} \sum_{k=0}^\infty tr \left( \hat{F}^{(k,m)} X^{(k,m)} \right) 
\end{equation}
where the solution matrices $X^{(k,m)}$ are in the convex hull of 
the matrices defined in (\ref{eq:directprod:aligncluster:Fourier:Block:structure1}).

The relaxation of the form (\ref{eq:NUG:convexrelax}) is 
\begin{equation}
\begin{array}{lll}
\underset{X^{(k,m)}}{\text{maximize}} & \sum_{m=0}^{M-1} \sum_{k=0}^{\infty} tr \left( \hat{F}^{(k,m)} X^{(k,m)} \right) & \\
{\text{subject to}}  & X^{(k,m)} \succeq 0 & \forall k,m \\
                     & X_{ii}^{(k,m)} = 1 & \forall k,m,i \\
                     & \sum_{k,m} tr \left(  \psi_{k,m}^*\left((g,a)\right) X^{(k,m)}_{ij} \ge 0 \right)  & \forall i,j ~,~ \forall (g,a) \in \mathcal{G} \times \mathbb{Z}_M \\
                     & X_{ij}^{(0,0)} = 1  & \forall  i,j \\
                     &  X^{(k,m)}_{ij} \ge -\frac{1}{M-1} & \forall m>0, ~~\forall i,j. \\
\end{array}
\end{equation}

In the following sections, we turn our attention to the ambiguities and symmetries
in $X^{(k,m)}$ of the convexified formulation (\ref{eq:directprod:aligncluster:convex}).

\subsection{The $0$ Order Representation of Alignment, and the Clustering Label Ambiguity}\label{sec:classamb}

As discussed in Section \ref{sec:avg}, when there is ambiguity in the solution
of the NUG, it is manifested as convex combinations of solutions in the 
covexified formulation (\ref{eq:directprod:aligncluster:convex}).
As discussed in Section \ref{sec:cluster}, there is ambiguity in the assignment of class labels which leads to symmetries in the NUG for the clustering problem.

We observe that the irreducible representations $\psi_{0,m}$ of $\mathcal{G} \times \mathbb{Z}_M$, enumerated in the first row in Table \ref{tb:directprod:aligncluster:ireps},
are simply the irreducible representations of $\mathbb{Z}_M$ which appear in the 
max-k-cut problem, as are the coefficients of the expansion of $f_{ij}$.  
Therefore, the same argument used in Section \ref{sec:cluster} can be used here to identify the desired form of the first row in the solution of the convex simultaneous alignment and classification problem (\ref{eq:directprod:aligncluster:convex}).
In fact, the same argument applies to all rows, which can be averaged in the same way;
the form of the averaged solution of each block $X^{(k,m)}_{ij}$ is summarized in Table \ref{tb:clusteralign:avgclust}, for the two cases: either $i$ and $j$ are in the same class (a),
or they are in different classes (b).

\begin{table}[h]
{\tiny
    \centering 
    \begin{minipage}[t]{0.5\linewidth}
\vspace{0pt}
\subfloat[$a_i = a_j$][$a_i = a_j$]{ 
\begin{tabular}{ l || c | c | c | c}
 $X^{(k,m)}$             & $m=0$      & $m=1$      & $\cdots$    & $m=M-1$      \\	
    \hline
    \hline
     $k=0$ & $1$ & {\color{red}$1$}       &  {\color{red}$\cdots$}    &  {\color{red}$1$}     \\
    \hline
      $k=1$ & $X^{(1,0)}_{ij} $ & {\color{red}$X^{(1,0)}_{ij} $}  & {\color{red}$\cdots$}      & {\color{red}$X^{(1,0)}_{i,j}$}     \\
    \hline
     $k=2$ & {$X^{(2,0)}_{i,j}$} & {\color{red}$X^{(2,0)}_{i,j}$}  & {\color{red}$\cdots$}      & {\color{red}$X^{(2,0)}_{i,j}$}    \\
    \hline 
     \vdots & \vdots & $\ddots$ & \vdots \\
\vspace{0.1cm}
  \end{tabular}
    }
  \end{minipage}
\begin{minipage}[t]{0.5\linewidth}
\vspace{0pt}
  \subfloat[$a_i \ne a_j$][$a_i \ne a_j$]{
\begin{tabular}{ l || c | c | c | c}			
 $X^{(k,m)}$             & $m=0$      & $m=1$      & $\cdots$    & $m=M-1$      \\	
    \hline
    \hline
     $k=0$ & $1$ & {\color{red}$-\frac{1}{M-1}$}       &  {\color{red}$\cdots$}    &   {\color{red}$-\frac{1}{M-1}$}     \\
    \hline
     $k=1$ & $X^{(1,0)}_{ij} $ & {\color{red}$-\frac{X^{(1,0)}_{ij}}{M-1}$}  & {\color{red}$\cdots$}      & {\color{red}$-\frac{X^{(1,0)}_{ij}}{M-1}$}     \\
    \hline
     $k=2$ & {$X^{(2,0)}_{i,j}$} & {\color{red}$-\frac{X^{(2,0)}_{i,j}}{M-1}$}  & {\color{red}$\cdots$}      & {\color{red}$-\frac{X^{(2,0)}_{i,j}}{M-1}$}     \\
    \hline 
     \vdots & \vdots & $\ddots$ & \vdots \\
  \end{tabular}
 }
  \end{minipage} 
}
  \caption{The desired form of blocks $X^{(k,m)}_{ij}$ of  $X^{(k,m)}$, corresponding to (a) same, and (b) distinct classes}\label{tb:clusteralign:avgclust}
  \end{table}

\subsection{Inter-Class Invariance}\label{sec:interclass}

In addition to the class label ambiguity, there is another type of ambiguity which emerges in the simultaneous clustering and alignment product NUG. 
We observe that the solution is invariant to a $\mathcal{G}$ group action on one class (without applying the same action to the other classes, so this is not a group action of $\mathcal{G} \times \mathcal{A}$).
\begin{lemma}\label{lem:interclass}
Let  $a_1,\ldots,a_n \in \mathbb{Z}_M $, $ g_1,\ldots,g_n \in \mathcal{G}$ and  $ \tilde{g}_1,\ldots,\tilde{g}_n \in \mathcal{G}$. Suppose that  $a \in \mathbb{Z}_M$ and $g \in \mathcal{G}$ are some arbitrary class and rotation, and 
suppose that 
\begin{equation}
  \tilde{g_i}=\begin{cases} \begin{array}{lll} g_ig &:& a_i = a \\ g_i &:& \text{otherwise}. \end{array} \end{cases}
\end{equation}
Then, the objective value in (\ref{eq:prob:aligncluster}) is the same for $ g_1,\ldots,g_n \in \mathcal{G}$ and $ \tilde{g}_1,\ldots,\tilde{g}_n \in \mathcal{G}$:
\begin{equation} 
\sum_{m=0}^{M-1} \sum_{\begin{subarray}{c} i,j: \\ a_i=m \\ a_j=m \end{subarray}} f_{ij}(g_i g_j^{-1}) = \sum_{m=0}^{M-1} \sum_{\begin{subarray}{c} i,j: \\ a_i=m \\ a_j=m \end{subarray}} f_{ij}(\tilde{g}_i  \tilde{g}_j^{-1}).
\end{equation}
In other words, if  $a_1,\ldots,a_n$, $ g_1,\ldots,g_n$ is a solution of (\ref{eq:prob:aligncluster}), 
then so is $a_1,\ldots,a_n$, $ \tilde{g}_1,\ldots,\tilde{g}_n$.
\end{lemma}
\begin{proof}
  For any $m \ne a$, $\tilde{g}_i=g_i$, so that 
\begin{equation}\label{eq:pf:interclassrot100}
\sum_{\begin{subarray}{c} i,j: \\ a_i=m \\ a_j=m \end{subarray}} f_{ij}(\tilde{g}_i \tilde{g}_j^{-1}) = \sum_{\begin{subarray}{c} i,j: \\ a_i=m \\ a_j=m \end{subarray}} f_{ij}({g}_i  {g}_j^{-1}) .
\end{equation}
  For $m=a$, we have 
\begin{equation}
f_{ij}(\tilde{g}_i  \tilde{g}_j^{-1}) = f_{ij}( ({g}_ig)  ({g}_jg)^{-1}) = f_{ij}({g}_i g g^{-1} {g}_j^{-1}) = f_{ij}({g}_i {g}_j^{-1})
\end{equation} 
so that (\ref{eq:pf:interclassrot100}) holds for $m=a$ as well.
\end{proof}

It follows that when $a_i \ne a_j$, we may average over all the possible inter-class alignment. 
By (\ref{eq:haar:avg}), using the Haar measure for the possible alignments yields 
$0$ for all elements with $k \ne 0$. 
The form of the averaged solution of each block $X^{(k,m)}_{ij}$ is summarized in Table \ref{tb:clusteralign:avginterclust}, for the two cases: either $i$ and $j$ are in the same cluster,
or they are in different clusters.
\begin{table}[h]
{\tiny
    \centering 
    \begin{minipage}[t]{0.5\linewidth}
\vspace{0pt}
  \subfloat[$a_i = a_j$][$a_i = a_j$]{ 
  \begin{tabular}{ l || c | c | c | c}
 $X^{(k,m)}$             & $m=0$      & $m=1$      & $\cdots$    & $m=M-1$      \\	
    \hline
    \hline
     $k=0$ & $1$ & {\color{red}$1$}       &  {\color{red}$\cdots$}    &  {\color{red}$1$}     \\
    \hline
      $k=1$ & $X^{(1,0)}_{ij} $ & {\color{red}$X^{(1,0)}_{ij} $}  & {\color{red}$\cdots$}      & {\color{red}$X^{(1,0)}_{i,j}$}     \\
    \hline
     $k=2$ & {$X^{(2,0)}_{i,j}$} & {\color{red}$X^{(2,0)}_{i,j}$}  & {\color{red}$\cdots$}      & {\color{red}$X^{(2,0)}_{i,j}$}    \\
    \hline 
     \vdots & \vdots & $\ddots$ & \vdots \\
  \end{tabular}
  }
\end{minipage}
\begin{minipage}[t]{0.5\linewidth}
\vspace{0pt}
  \subfloat[$a_i \ne a_j$][$a_i \ne a_j$]{
  \begin{tabular}{ l || c | c | c | c}			
 $X^{(k,m)}$             & $m=0$      & $m=1$      & $\cdots$    & $m=M-1$      \\	
    \hline
    \hline
     $k=0$ & $1$ & {\color{red}$-\frac{1}{M-1}$}       &  {\color{red}$\cdots$}    &   {\color{red}$-\frac{1}{M-1}$}     \\
    \hline
     $k=1$ & {\color{red}$0$} & {\color{red}$0$}  & {\color{red}$\cdots$}      & {\color{red}$0$}     \\
    \hline
     $k=2$ & {\color{red}$0$} & {\color{red}$0$}  & {\color{red}$\cdots$}      & {\color{red}$0$}     \\
    \hline 
     \vdots & \vdots & $\ddots$ & \vdots \\
\vspace{0.01cm}
  \end{tabular}}
\end{minipage}
}
  \caption{The desired form of blocks $X^{(k,m)}_{ij}$ of  $X^{(k,m)}$, after averaging inter-class rotations (Lemma \ref{lem:interclass}) }\label{tb:clusteralign:avginterclust}
  \end{table}

%
\section{Algorithms}\label{sec:alg}

Substituting the results of Section \ref{sec:nugprod} into (\ref{eq:NUG:convexrelax})
we obtain the following SDP:
  \begin{equation}\label{eq:alg:convexrelax}
    \begin{array}{lll}
      \underset{ \{X^{(k,m)}\}_{k,m} }{\argmin} & \sum_{k=0}^\infty\sum_{m=0}^{M-1} tr \left( \hat{F}^{(k,m)} X^{(k,m)} \right) & \\
      {\text{subject to}}  & X^{(k,m)} \succeq 0 & \forall k,m \\
                     & X_{ii}^{(k,m)} = I_{d_k \times d_k} & \forall k,m,i \\
                     & \sum_{k=0}^\infty \sum_{m=0}^{M-1} d_k tr \left( \psi^*_{k,m}((g,a)) X^{(k)}_{ij}  \right) \ge 0 & \forall  i,j  \\
      & & \forall g,a\in\mathcal{G} \times \mathbb{Z}_M \\
                     & X_{ij}^{(0,0)} = 1  & \forall  i,j  \\
  \end{array}
  \end{equation}

  The coefficient in the matrix $\hat{F}^{(k,m)}$ can be obtained from the original alignment problem, when no clustering is required; suppose that the coefficients in that problem are $\hat{F}^{(k)}$, then for all $k$ and $m$, the coefficients $\hat{F}^{(k,m)}$ are
\begin{equation}
  \hat{F}^{(k,m)} = \frac{1}{M}\hat{F}^{(k)} .
\end{equation}

We observe that due to  the structure discussed in Section \ref{sec:classamb} and Section \ref{sec:interclass}, regardless of whether $a_i = a_j$ or $a_i \ne a_j$, 
\begin{equation}\label{eq:alg:mateq}
\begin{array}{ll}
{ X^{(0,m)}=X^{(0,1)}} & \forall m \ne 0 \\
{ X^{(k,m)}=X^{(k,0)}} & \forall k \ne 0 ~,~ \forall m .
\end{array}
\end{equation}
Taking these observations into account,  (\ref{eq:alg:convexrelax}) is reduced to
\begin{equation}\label{eq:alg:convexrelax:reduce1}
\begin{array}{lll}
\underset{X^{(k,m)}}{\argmin} & \sum_{m=0}^{M-1} \sum_{k=0}^{\infty} tr \left( \hat{F}^{(k,m)} X^{(k,m)} \right) & \\
{\text{subject to}}  &{ X^{(0,m)}=X^{(0,1)}} & \forall m \ne 0 \\
&{ X^{(k,m)}=X^{(k,0)}} & \forall k \ne 0 ~~,~~ \forall m \\
&{ X^{(0,m)}_{ij} \ge -\frac{1}{M-1}  }& \forall m>0 ~~,~~ \forall i,j \\ 
&X^{(k,m)} \succeq 0 & \forall k,m \\
&X_{ii}^{(k,m)} = 1 & \forall k,i \\
&\sum_{k,m} tr \left(  \psi_{k,m}^*\left((g,a)\right) X^{(k,m)}_{ij}  \right) \ge 0 & \forall i,j ~,~ \forall (g,a) \in \mathcal{G} \times \mathbb{Z}_M  \\
& X_{ij}^{(0,0)} = 1  & \forall  i,j .\\
\end{array}
\end{equation}

In fact, the requirement for non-negativity over $\mathcal{G} \times \mathbb{Z}_M$
is redundant, due to the following lemma.
\begin{lemma}
Suppose that $a \ne e$ (where $e$ is the identity element of $\mathbb{Z}_M$). 
If the other constraints in (\ref{eq:alg:convexrelax:reduce1}) are satisfied, 
then for all $i,j$,
\begin{equation}
\sum_{k,m} tr \left(  \rho_{k,m}^*\left((g,a)\right) X^{(k,m)}_{ij}  \right) \ge 0 
\end{equation}
for all $g \in \mathcal{G}$ and all $a \ne e$.
\end{lemma}
\begin{proof}
  Due to the other constraints in (\ref{eq:alg:convexrelax:reduce1}),
  for all $k > 0$, we have $X^{(k,0)}=X^{(k,1)}=\ldots=X^{(k,m)}$, so that for all $k > 0$
\begin{equation}
\sum_{m=0}^{M-1} tr \left(  \psi_{k,m}^*\left((g,a)\right) X^{(k,m)}_{ij}  \right) =
 tr \left(  \rho_{k}^*\left(g\right) X^{(k,0)}_{ij}  \right) \sum_{m=0}^{M-1} \eta_m(a) = 0
\end{equation}
where the last step is due to the fact that $\sum_{m=0}^{M-1} \eta_m(a) =0$ for $a$ that are not the identity.

For $k=0$, we have $X_{ij}^{(0,0)} = 1$ and $ -\frac{1}{M-1} \le X^{(0,m)} \le 1 $, so that
\begin{equation}
\sum_{m=0}^{M-1} tr \left(  \psi_{0,m}^*\left((g,a)\right) X^{(0,m)}_{ij}  \right) =
1 + \sum_{m=1}^{M-1}  \eta_m\left(a\right) X^{(0,m)}_{ij}  \ge 1 - (M-1)/(M-1) =0
\end{equation}

\end{proof}

Using this lemma,  (\ref{eq:alg:convexrelax:reduce1}) is reduced to 
\begin{equation}\label{eq:alg:convexrelax:reduce2}
\begin{array}{lll}
\underset{X^{(k,m)}}{\argmin} & \sum_{m=0}^{M-1} \sum_{k=0}^{\infty} tr \left( \hat{F}^{(k,m)} X^{(k,m)} \right) & \\
{\text{subject to}}  &{ X^{(0,m)}=X^{(0,1)}} & \forall m \ne 0 \\
&{ X^{(k,m)}=X^{(k,0)}} & \forall k \ne 0 \\
&{  X^{(0,m)}_{ij} \ge -\frac{1}{M-1} }& \forall m>0 ~~,~~ \forall i,j \\
&X^{(k,m)} \succeq 0 & \forall k,m \\
&X_{ii}^{(k,m)} = 1 & \forall k,i \\
&\sum_{k,m} tr \left(  \rho_{k,m}^*\left((g,e)\right) X^{(k,m)}_{ij}  \right) \ge 0 & \forall i,j ~,~ \forall g \in \mathcal{G}  \\
& X_{ij}^{(0,0)} = 1  & \forall  i,j \\
\end{array}
\end{equation}
where $e$ is the identity element of $\mathbb{Z}_M$.

\subsection{Controlling Class Size}

When the size of the classes is known to be equal, the constraint (\ref{eq:equidistrib:sum0}) of Section \ref{sec:sizecontrol} is added to the SDP. Considering all the symmetries, the constraint takes the form 
\begin{equation}\label{eq:alg:size}
   \sum_{j} X^{(0,1)}_{ij} = 0   ~~~\forall i  .
\end{equation}


%
%
%
\subsection{Variable and Constraints Accounting }

The purpose of this section is to discuss the number of free variables remaining in the formulation (\ref{eq:alg:convexrelax:reduce2}), and the number of constraints. 
We note that the only remaining matrix variables are $X^{(0,1)}$ and $X^{(1,0)},X^{(2,0)},X^{(3,0)},\ldots$. The matrix $X^{(0,0)}$ is the trivial all ones matrix, and every other matrix is set to be equal to the appropriate matrix of those listed above (see (\ref{eq:alg:mateq})).
We observe that the matrix $X^{(0,1)}$ has exactly the same form as the matrix $Y$ in the max-cut classification SDP, and the constrains on it are similar. 
The matrices $X^{(1,0)},X^{(2,0)},X^{(3,0)}..$ have the same form as the matrices 
$X^{(1)},X^{(2)},X^{(3)},\ldots$ in the alignment problem, and also have similar constraints.
In other words, loosely speaking the number of variables and constraints in the product NUG discussed here is similar to the sum of those in the separate classification problem and those in the alignment problem, which is much smaller than the number of variables and constraints of the formulation (\ref{eq:alg:convexrelax}) which we obtained before considering the symmetries.

%
\section{Experimental Results}\label{sec:res}

In this section we present experiments with the simplified case of alignment and clustering of noisy functions on SO(2) (also discussed in Section \ref{sec:ex}). 
We generated $4$ complex valued prototype functions over SO(2), the functions are low-bandwidth, represented by $11$ coefficients in the Fourier domain. 
For each prototype function we generated $15$ copies, each copy was shifted randomly on SO(2), and random noise was added to each of the shifted copies, yielding a dataset $\{s_i\}_{i=1}^{n}$ of $n=60$ signals. 
The problem is now to align and cluster the signals in the dataset.

This problem is simpler than  the Cryo-EM problem, but it contains the key components and allows us to construct a benchmark. We observe that the auto-correlation and bispectrum \cite{sadler1992shift} 
of signals are invariant to rotations; therefore, in the absence of noise, we can compute the auto-correlation or bispectrum of each signal in our dataset, and use these as ``signatures'' to cluster the signals. 
In the presence of noise, these signatures are distorted, leading to possible errors in clustering. We experimented with both auto-correlation and bispectrum; since the results were very similar in the two cases, and since bispectrum has certain theoretical advantages, we present the results for bispectrum here.

We implemented the SDP in (\ref{eq:alg:convexrelax:reduce2}) with balanced classes (\ref{eq:alg:size}) in Matlab, using CVX \cite{cvx,gb08}. 
For every pair of signals $s_i$ and $s_j$ we compute $f_{ij}$: 
\begin{equation}
  f_{ij}(g) = \| s_i - g \circ s_j \| ,
\end{equation} 
where $g \circ  s_j$ is the signal $s_j$ rotated by $g$. The rotation is implemented by multiplication by the appropriate phase in the Fourier domain. 
We construct the $n \times n$ matrices of coefficients $\hat{F}^{(k)}$ (the matrices for multireference alignment without classification); the elements in the $i,j$ position in the $k$ matrix is the $k$ element in the DFT of $f_{ji}$.
\begin{equation}
  \hat{f}^{(\cdot)}_{ij} = \mathcal{F}\left(f_{ji}\right) . 
\end{equation}
The non-negativity constraint is implemented using the Fejer kernel (see \cite{bandeira2015non}).

In order to study the performance of the algorithm, we focused on the classification aspect, which can be compared to clustering obtained through the use of bispectrum ``signatures.''
We computed the bispectrum of each signal in the dataset and also solved the SDP for the product NUG of this dataset.
For simplicity, we used the simple k-means to cluster the signals: first by the bispectral signature of each signal, and then by the columns of the matrix $X^{(0,1)}$ obtained by the SDP. 
For simplicity, we did not enforce equal cluster sizes in the k-means. 
We measured the fraction of signals that were misclassified (the clusters are recovered only up to permutation: even if the k-means find the correct clusters, the class ``labels'' are assigned arbitrarily. We computed the minimum error over all permutations of class labels).

We repeated the experiment $20$ times for every noise level. The results are presented in Figure \ref{fig:res1}. The experiment demonstrates that the product NUG achieves considerably better classification results in the presence of noise.

\begin{remark}
In the Cryo-EM problem, the images which we wish to align are different projections of the molecule $\mathcal{X}$.
While bispectrum and auto-correlation have been used to find images from the same plane (see \cite{zhao2014rotationally}), these signatures are not invariant to projections. 
Therefore, in the Cryo-EM problem, these signatures cannot be used for classification, so they do not provide an alternative for the product NUG discussed here. 

In other words, although the product NUG achieves better results than invariant signature based clustering in these experiments, its true importance is in cases where such alternative methods cannot be used. 
\end{remark}

\begin{figure}[tp]	
    \centering
    {\includegraphics[width=60mm,angle=0]{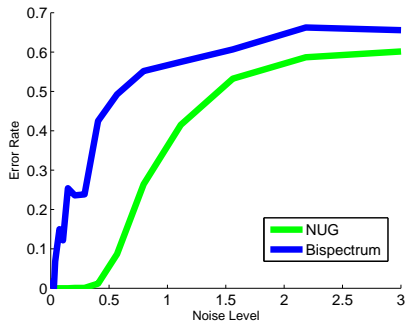}}
    \caption{Classification error vs. noise level, 4 balanced clusters}\label{fig:res1}
\end{figure}

%
\section{Summary and Future Work}\label{sec:conclusion}

The problem of simultaneous alignment and classification has been formulated as a Non-Unique Game, and an algorithm has been presented for solving the a convex relaxation of the problem. 
The algorithm has been demonstrated for the case of simultaneous alignment and classification of mixed signals on SO(2); and it is currently being adapted for the heterogeneity problem of Cryo-EM. 
It should be noted that SDPs like the one proposed here are difficult to scale using off-the-shelf solvers to very large problems, such as alignment of hundreds of thousands of images produced in modern Cryo-EM experiments. Nevertheless, special purpose solvers provide more scalability, the SDPs offer certificates of global optimality of solutions found using other approaches in some circumstances, they provide a benchmark for approximate optimizations, and they can be applied to reduced datasets (e.g. class averages of images). 

The approach discussed here can  be generalized to the case of continuous heterogeneity, where the molecules are not classified to distinct classes, but rather lie on continuum of states that can be parametrized (alternatively, the states are distinct, but related to some degree).
In this case, we follow similar ideas to those in this manuscript, however there are some additional details that require considerations in the choice of underlying groups and the structure of $f_{ij}$;
this case will be discussed in more detail in a future paper.

As discussed in Section \ref{sec:sizecontrol}, there are several variations of the control over the size of clusters. Furthermore, the same ideas can be used to control the distribution of the recovered rotation angles (for example, when the images can be assumed to come from approximately uniform distribution over SO(3)).

%
\section*{Acknowledgments}

The authors would like to thank Joakim And\'{e}n, Afonso Bandeira, Tejal Bhamre, Yutong Chen and Justin Solomon for their help. 

The authors were partially supported by Award Number
R01GM090200 from the NIGMS, FA9550-12-1-0317 and
FA9550-13-1-0076 from AFOSR, LTR DTD 06-05-2012 from the Simons Foundation, and the Moore Foundation Data-Driven Discovery Investigator Award.

Part of the work by RRL was done while visiting the Hausdorff Research Institute for Mathematics, as part of the Mathematics of Signal Processing trimester.


\bibliography{bib06}{}
\bibliographystyle{ieeetr}


%

\end{document}